\documentclass{article}

\usepackage[preprint]{neurips_2023}

\usepackage{microtype}
\usepackage{graphicx}
\usepackage{float}
\usepackage{subfigure}
\usepackage{booktabs} 

\usepackage{thmtools} 
\usepackage{thm-restate}

\PassOptionsToPackage{hyphens}{url}\usepackage{hyperref}
\usepackage{enumitem}

\usepackage{algorithm}
\usepackage{algorithmicx}

\usepackage{hyperref}
\usepackage{multicol}
\usepackage{multirow}

\usepackage{amsmath}
\usepackage{amssymb}
\usepackage{mathtools,mathrsfs}
\usepackage{amsthm,amsfonts}
\newtheorem*{theorem*}{Theorem}
\usepackage{setspace}

\usepackage[capitalize,noabbrev]{cleveref}

\theoremstyle{plain}
\newtheorem{theorem}{Theorem}[section]

\newtheorem{lemma}[theorem]{Lemma}

\theoremstyle{definition}

\theoremstyle{remark}

\usepackage[textsize=tiny]{todonotes}

\usepackage{xcolor}
\usepackage{eqparbox}
\renewcommand{\algorithmiccomment}[1]{\quad\eqparbox{COMMENT}{\color{blue} $\triangleright\,\,$ #1}}

\newif\ifarxiv
\ifarxiv\else
\arxivfalse
\fi

\usepackage{xspace}
\usepackage{multicol}
\raggedcolumns

\usepackage{xspace}         
\usepackage{amssymb}
\usepackage{titlesec}

\newif\ifcomm
\ifcomm
\else
\commfalse
\fi

\ifcomm
	\newcommand{\mycomm}[3]{{\footnotesize{{\color{#2} \textbf{[#1: #3]}}}}}
	\newcommand{\CRdel}[1]{\textcolor{red}{\sout{#1}}}
\else
    \newcommand{\mycomm}[3]{}
    \newcommand{\CRdel}[1]{}
\fi

\newcommand{\RBB}[1]{\mycomm{Ran}{purple}{#1}} 
\newcommand{\SV}[1]{\mycomm{Shay}{brown}{#1}} 
\newcommand{\AP}[1]{\mycomm{Amit}{red}{#1}} 
\newcommand{\ran}[1]{\RBB{#1}}

\newcommand{\T}[1]{\noindent\textbf{#1}}

\DeclareMathOperator*{\minimize}{minimize}

\newcommand{\sender}{client\xspace}
\newcommand{\Sender}{Client\xspace}
\newcommand{\senders}{clients\xspace}
\newcommand{\receiver}{server\xspace}
\newcommand{\Receiver}{Server\xspace}

\newcommand{\vnmse}{\ensuremath{\mathit{vNMSE}}\xspace}
\newcommand{\nmse}{\ensuremath{\mathit{NMSE}}\xspace}

\newcommand{\alg}{{QUIC-FL}\xspace}

\newcommand{\norm}[1]{\left\lVert#1\right\rVert}
\newcommand{\parentheses}[1]{\left(#1\right)}
\newcommand{\angles}[1]{\left\langle#1\right\rangle}
\newcommand{\brackets}[1]{\left[#1\right]}
\newcommand{\set}[1]{\left\{#1\right\}}
\newcommand{\abs}[1]{\left|#1\right|}

\newcommand*\samethanks[1][\value{footnote}]{\footnotemark[#1]}
\newcommand{\E}{\mathbb{E}}

\makeatletter
\newcommand\footnoteref[1]{\protected@xdef\@thefnmark{\ref{#1}}\@footnotemark}
\makeatother

\usepackage[utf8]{inputenc} 
\usepackage[T1]{fontenc}    
\usepackage{hyperref}       
\usepackage{url}            
\usepackage{booktabs}       
\usepackage{amsfonts}       
\usepackage{nicefrac}       
\usepackage{microtype}      
\usepackage{xcolor}         

\title{\resizebox{0.9996\textwidth}{!}{\alg: Quick Unbiased Compression for Federated Learning}}

\author{
\hspace{-5mm}Ran Ben-Basat \samethanks[1]\\
\hspace{-5mm}University College London\\
\hspace{-5mm}\texttt{r.benbasat@ucl.ac.uk}
\And
\hspace{7mm}Shay Vargaftik \thanks{Equal Contribution.}\\
\hspace{7mm}VMware Research \\
\hspace{7mm}\texttt{shayv@vmware.com}
\AND
\hspace{2mm}Amit Portnoy \samethanks[1]\\
\hspace{2mm}Ben-Gurion University of the Negev \\
\hspace{2mm}\texttt{amitport@post.bgu.ac.il}
\And
Gil Einziger \\
Ben-Gurion University of the Negev \\
\texttt{gilein@bgu.ac.il}
\AND
\hspace{6mm}Yaniv Ben-Itzhak \\
\hspace{6mm}VMware Research \\
\hspace{6mm}\texttt{ybenitzhak@vmware.com}
\And
\hspace{5mm}Michael Mitzenmacher \\
\hspace{5mm}Harvard University \\
\hspace{5mm}\texttt{michaelm@eecs.harvard.edu}
}

\begin{document}

\maketitle

\vspace*{-5mm}
\begin{abstract}
\vspace*{-1mm}
Distributed Mean Estimation (DME), in which $n$ clients communicate vectors to a parameter server that estimates their average, is a fundamental building block in communication-efficient federated learning.
In this paper, we improve on previous DME techniques that achieve the optimal $O(1/n)$ Normalized Mean Squared Error (NMSE) guarantee by asymptotically improving the complexity for either encoding or decoding (or both). 
To achieve this, we formalize the problem in a novel way that allows \mbox{us to use off-the-shelf mathematical solvers to design the quantization.}
\end{abstract}

\vspace*{-4mm}
\section{Introduction} \label{sec:intro}
\vspace*{-1mm}

Federated learning~\cite{mcmahan2017communication,kairouz2019advances}, 
is a technique to train models across multiple clients without having to share their data.
During each training round, the participating clients send their model updates (hereafter referred to as gradients) to a parameter server that calculates their mean and updates the model for the next round. 
Collecting the gradients from the participating clients is often communication-intensive, which implies that the network becomes a bottleneck. Thus, many works focus on reducing communication overheads utilizing compression. Typically, the clients reduce bandwidth by sending compact approximations of their gradients, which implies that the parameter server only achieves an approximation of the mean. 
Such methods offer tradeoffs between the required bandwidth, \mbox{computational efficiency, and estimation accuracy.} 

Formally, the \emph{Distributed Mean Estimation (DME)} problem is defined as follows.
Consider $n$ clients with $d$-dimensional vectors (e.g., gradients) to report; each client sends an approximation of its vector to a parameter server (hereafter referred to as `server') which estimates the vectors' mean, e.g., see \citet{pmlr-v70-suresh17a,konevcny2018randomized,vargaftik2021drive,davies2021new,EDEN}). 
We briefly survey the most relevant and recent related works for DME.  Common to these techniques is that they preprocess the input vectors into a different representation that allows for better lossy compression, generally through quantization of the coordinates.  

For example, in~\citet{pmlr-v70-suresh17a}, each client, in $O(d \cdot \log d)$ time, uses a Randomized Hadamard Transform (RHT) to preprocess its vector and then applies stochastic quantization. 
The transformed vector has a smaller coordinate range (in expectation), which reduces the quantization error. 
The server then aggregates the transformed vectors before applying the inverse transform to estimate the mean, for a total of $O(n\cdot d + d\cdot\log d)$ time.
Such a method has a Normalized Mean Squared Error (\nmse) that is bounded by $O\parentheses{\log d /n}$ using $O(1)$ bits per coordinate. Hereafter, we refer to this method as `Hadamard'.
This work also suggests an alternative method that uses entropy encoding to achieve an NMSE of $O(1/n)$, which is optimal. However, entropy encoding is a compute-intensive process \mbox{that does not efficiently translate to GPU execution, resulting in a slow decode time.}

A different approach to DME computes the Kashin's representation~\cite{lyubarskii2010uncertainty} of a client's vector $\overline x$ before applying quantization~\cite{caldas2018expanding,safaryan2020uncertainty}.
Intuitively, this replaces the $d$-dimensional input vector by $O(d)$ coefficients, each bounded by $O({\norm{\overline x}_2}/{\sqrt d})$. Applying quantization to the coefficients instead of the original vectors allows the \receiver to estimate the mean using $O(1)$ bits per coordinate with an $O(1/ n)$ \nmse. However, computing the coefficients requires applying multiple RHTs, asymptotically slowing down its \emph{encoding} time from Hadamard's $O(d\cdot \log d)$ to $O(d\cdot\log d\cdot\log(n\cdot d))$.

The works of~\citet{vargaftik2021drive,EDEN} transform the input vectors in the same manner as~\citet{pmlr-v70-suresh17a}, but with two differences: (1) clients must use independent transforms; (2) clients use deterministic (\emph{biased}) quantization, derived using existing information-theoretic tools like the Lloyd-Max quantizer, on their transformed vectors. Interestingly, the server still achieves an \emph{unbiased} estimate of each client's input vector after multiplying the estimated vector by a real-valued `scale' (that is sent by the client) and applying the inverse transform. Using uniform random rotations, which RHT approximates, such a process achieves $O(1/n)$ \nmse and is empirically more accurate than Kashin's representation. With RHT, their encoding complexity is $O(d\cdot\log d)$, matching that of~\citet{pmlr-v70-suresh17a}. 
However, since the clients transform their vectors independently of each other (and thus the server must invert their transforms individually, i.e., perform $n$ inverse transforms), the decode time is \mbox{asymptotically increased to $O(n\cdot d\cdot \log d)$ compared to Hadamard's $O(n\cdot d + d\cdot \log d)$.}

While the above methods suggest aggregating  the gradients directly using DME, recent works leverage it as a building block. For example, in EF21~\cite{richtarik2021ef21}, each client sends the compressed difference between its local gradient and local state, and the server estimates the mean to update the global state.
Similarly, DIANA~\cite{mishchenko2019distributed} uses DME to estimate the average gradient difference. Thus, better DME techniques can improve their performance (see Appendix~\ref{app:ef21_topk_qfl}).
{We defer further discussion of frameworks that use DME as a building block to Appendix~\ref{app:extended_RW}.


\looseness=-1
In this work, we present \textbf{Q}uick \textbf{U}nb\textbf{i}ased \textbf{C}ompression for \textbf{F}ederated \textbf{L}earning  (\alg), a DME method with $O(d\cdot \log d)$ encode and $O(n\cdot d + d\cdot \log d)$ decode times, and the optimal $O(1/n)$ \nmse.
As summarized in~\cref{tbl:asymptotics}, \alg{} asymptotically improves over the best encoding and/or 
\mbox
{decoding times of techniques with this \nmse guarantee. }

In \alg, 
each client applies RHT and quantizes its transformed vector using an \emph{unbiased} method we develop to minimize the quantization error. Critically, all clients use \emph{the same} transform, thus allowing the server to aggregate the results before applying a single inverse transform. 
\alg's quantization features two new techniques; first, we present Bounded Support Quantization (BSQ), where clients send a small fraction of their largest (transformed) coordinates exactly, thus minimizing the difference between the largest quantized coordinate and the smallest one and thereby the quantization error.
Second, we design a near-optimal distribution-aware unbiased quantization. To the best of our knowledge, such a method is not known in the information-theory 
\mbox
{literature and may be of independent interest.}

\begin{table*}[]
\centering
\resizebox{\textwidth}{!}{
\begin{tabular}{|l|l|l|l|}
\hline
\textbf{Algorithm}                                                                                     & \textbf{Enc. complexity}                         & \textbf{Dec. complexity}             & \textbf{NMSE}                           \\ \hline
\textbf{QSGD} \cite{NIPS2017_6c340f25}                                                   & $O(d)$                                  & $O(n\cdot d)$               & $O(d/n)$                       \\ \hline
\textbf{Hadamard} \cite{pmlr-v70-suresh17a}                                               & $O(d\cdot \log d)$                      & $O(n\cdot d+d\cdot \log d)$ & $O(\log d /n)$                 \\ \hline
\textbf{Kashin} \cite{caldas2018expanding,safaryan2020uncertainty} & $O(d\cdot \log d\cdot \log (n\cdot d))$ & $O(n\cdot d+d\cdot \log d)$ & $O(1 /n)$                      \\ \hline
\textbf{EDEN} \cite{EDEN}                                                                 & $O(d\cdot \log d)$                      & $O(n\cdot d \cdot \log d)$  & $O(1 /n)$                      \\ \hline
{\textbf{QUIC-FL (New)}}                                                                                 & $O(d\cdot \log d)$                      & $O(n\cdot d+d\cdot \log d)$ & \multicolumn{1}{l|}{$O(1 /n)$} \\ \hline 
\end{tabular}
}
\vspace*{-3.2075mm}
\caption{
Unbiased DME algorithms' guarantees (without variable-length encoding; see Appendix~\ref{app:extended_RW}) using $b=O(1)$ bits per coordinate and using the Hadamard transform for rotation-based algorithms. }
\label{tbl:asymptotics}
\vspace*{-2.48058018509764mm}
\end{table*}

\begin{figure}
\centering
  \begin{center}
    \includegraphics[width=.61235799\columnwidth]{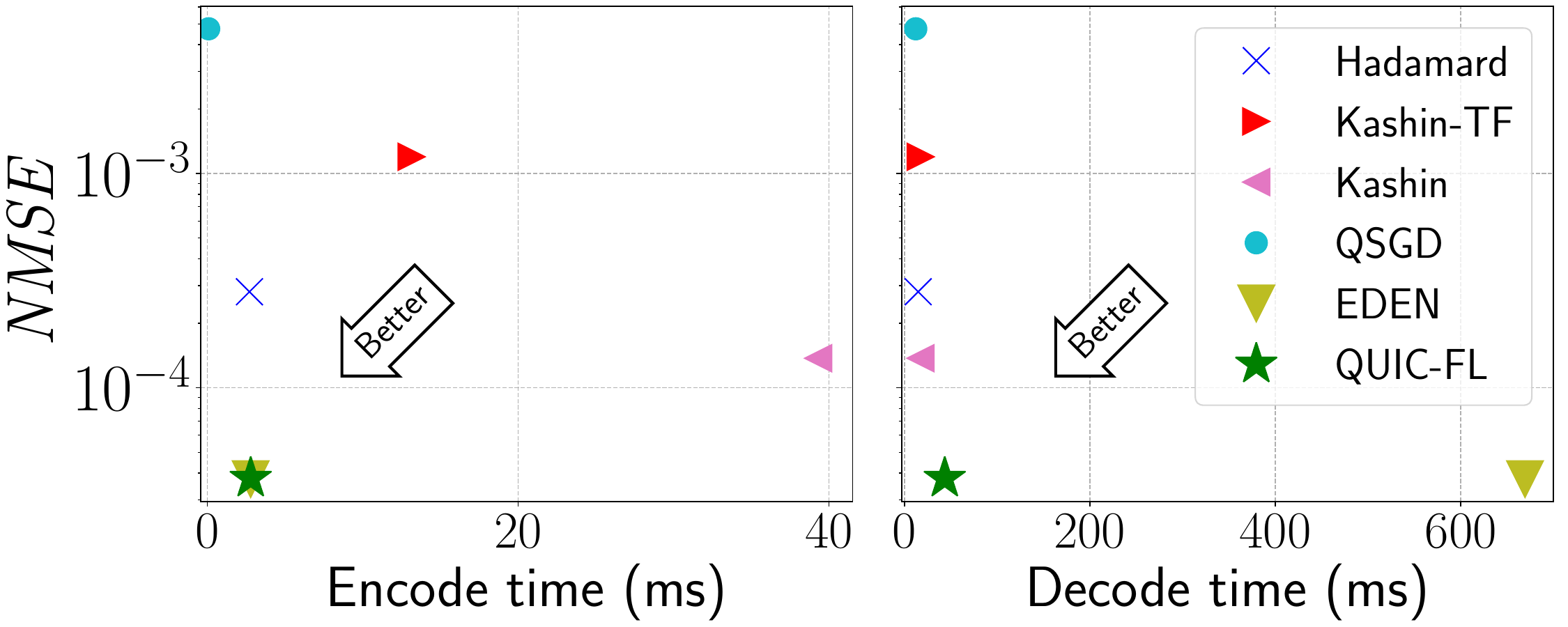}
  \end{center}
  \vspace*{-4.525mm}
  \caption{Normalized Mean Squared Error vs. processing time. \SV{maybe its better to split into two plots... encode and decode separately}}
  \label{fig:intro_example}
  \vspace*{-5.00963mm}
\end{figure}

{
We implement \alg in PyTorch~\cite{NIPS2019_9015} and TensorFlow~\cite{tensorflow2015-whitepaper} and evaluate it on different FL tasks (\cref{sec:evaluation}). We show that \alg can compress vectors with over 33 million coordinates within 44 milliseconds and is markedly more accurate than existing $O(n\cdot d)$ and $O(n\cdot d + d\cdot \log d)$ 
decode time \mbox{approaches} such as QSGD~\cite{NIPS2017_6c340f25},  Hadamard~\cite{pmlr-v70-suresh17a}, and Kashin~\cite{caldas2018expanding,safaryan2020uncertainty}.
Compared with DRIVE~\cite{vargaftik2021drive} and EDEN~\cite{EDEN}, \alg has a competitive NMSE while asymptotically improving the estimation time, as shown in \cref{fig:intro_example}. The figure illustrates the encode and decode times vs. NMSE for $b=4$ bits per coordinate, $d=2^{20}$ dimensions, and $n=256$ clients. Our code will be released as open source upon publication.
}

\SV{point to the extended related work here or in the relevant section...}


\section{Preliminaries}\label{sec:Preliminaries}

\paragraph{Notation.} Capital letters denote random variables (e.g., $I_c$) or functions (e.g., $T(\cdot)$); overlines denote vectors (e.g., $\overline x_c$); calligraphic letters stand for sets (e.g., $\mathcal X_b)$ with the exception of $\mathcal N$ and $\mathcal U$ that denote the normal and uniform distributions; and hats denote estimators (e.g., $\widehat {\overline x}_{\mathit{avg}}$).

\paragraph{Problems and Metrics.}
Given a nonzero vector ${\overline x\in\mathbb R^d}$, a vector compression protocol consists of a \sender that sends a message to a \receiver that uses it to estimate $\widehat {\overline x}\in\mathbb R^d$. The \emph{vector Normalized Mean Squared Error} (\vnmse) of the protocol is defined as $ \frac{\E\brackets{\norm{\widehat {\overline x}- \overline x}_2^2}}{\norm{\overline x}_2^2}~~.$


The above generalizes to Distributed Mean Estimation (DME), where each of $n$ \senders{} has a nonzero vector ${\overline x_c\in\mathbb R^d}$, where $c\in\set{0,\ldots,n-1}$,  that they compress and communicate to a \receiver. We are interested in minimizing {the \emph{Normalized Mean Squared Error} (\nmse), defined as $ \frac{\E\brackets{\norm{\widehat {\overline x}_{\mathit{avg}} - \frac{1}{n}\sum_{c=0}^{n-1} \overline x_c}_2^2}}{\frac{1}{n}\cdot\sum_{c=0}^{n-1}\norm{\overline x_c}_2^2}~~,$}
%
%
where $\widehat {\overline x}_{\mathit{avg}}$ is our estimate of the average $\frac{1}{n}\cdot \sum_{c=0}^{n-1}  \overline x_c$.
For unbiased algorithms and independent estimates, we that $\nmse=\vnmse/n$.

\vspace{-2mm}
\paragraph{Randomness.}
We use \textit{global} (common to all clients and the server) and \textit{client-specific} shared randomness (one client and server). 
\mbox{Client-only randomness is called \textit{private}.}
\vspace{-2mm}


\section{The \alg Algorithm} \label{sec:alg}

We first describe our design goals in \cref{sec:alg:design_goals}. Then, in \cref{sec:alg:bsq,sec:alg:opt_sq}, we successively present two new tools we have developed to achieve our goals, namely, \emph{bounded support quantization} and \emph{distribution-aware unbiased quantization}. In \cref{sec:alg:alg}, we present \alg's pseudocode and discuss its properties and guarantees. Finally, in \cref{sec:alg:scary_stuff_no_one_will_understand}, we overview additional optimizations.

\subsection{Design Goals}\label{sec:alg:design_goals}

We aim to develop a DME technique that requires less computational overhead while achieving the same accuracy at the same compression level as the best previous techniques.

As shown by recent works \cite{pmlr-v70-suresh17a,lyubarskii2010uncertainty,caldas2018expanding,safaryan2020uncertainty,vargaftik2021drive,EDEN}, a preprocessing stage that transforms each client's vector to a vector with a different distribution (such as applying a uniform random rotation or RHT) can lead to smaller quantization errors and asymptotically lower \nmse. However, in existing DME techniques that achieve the asymptotically optimal \nmse of $O(1/n)$, such preprocessing incurs a high computational overhead on either the clients (i.e., \citet{lyubarskii2010uncertainty,caldas2018expanding,safaryan2020uncertainty}) or the server (i.e., \citet{lyubarskii2010uncertainty,caldas2018expanding,safaryan2020uncertainty,vargaftik2021drive,EDEN}). The question is then how to preserve the appealing \nmse of $O({1}/{n})$ but reduce the computational burden?

In \alg, similarly to previous DME techniques, we use a preprocessing stage\footnote{For simplicity, we first consider a uniform random rotation and then, in \cref{sec:alg:scary_stuff_no_one_will_understand}, move to the computationally efficient RHT instead, while preserving the guarantees specified in~\cref{tbl:asymptotics}.} where each client applies a uniform random rotation on its input vector. After the rotation, the coordinates' distribution approaches independent normal random variables for high dimensions~\cite{vargaftik2021drive}. We use our knowledge of the resulting distribution to devise a fast and near-optimal unbiased quantization scheme that both preserves the appealing $O(1/n)$ \nmse guarantee and is asymptotically faster than existing DME techniques with similar \nmse guarantees.  A particularly important aspect of our scheme is that we can avoid decompressing each client's compressed vector at the server by having all clients use the same rotation (determined by shared randomness), so that the server can directly sum the compressed results and perform a single inverse rotation.  

\vspace{-1mm}
\subsection{Bounded support quantization}\label{sec:alg:bsq}

Our first contribution is the introduction of \emph{bounded support quantization} (BSQ).
For a parameter $p\in(0,1]$, we pick a threshold $t_p$ such that up to $d\cdot p$ values can fall outside $[-t_p,t_p]$.
BSQ separates the vector into two parts: the
small values in the range $[-t_p,t_p]$, and the remaining (large) values.  The large values are sent exactly (matching the precision of the input), whereas the small values are stochastically quantized and sent using a small number of bits each. This approach decreases the error of the quantized \mbox{values by bounding their support at the cost of sending a small number of values exactly.}

For the exactly sent values, we also need to send their indices. There are different ways to do so. For example, it is possible to encode these indices using $\log{\binom{d}{d\cdot p}}\approx d\cdot p\cdot \log(1/p)$ bits at the cost of higher complexity.  When the $d\cdot p$ indices are uniformly distributed (which will be essentially our case later), then delta coding methods can be applied (see, e.g., Section 2.3 of \citet{DBLP:journals/pvldb/VaidyaKCKMI22}).  Alternatively, we can send these indices without any additional encoding using $d\cdot p\cdot \lceil \log d \rceil$ bits (i.e., $\lceil \log d \rceil$ bits per transmitted index) or transmit a bit-vector with an indicator for each value whether it is exact or quantized. Empirically, sending the indices using $\lceil \log d \rceil$ bits each
without encoding is most useful, as $p\cdot \log d \ll 1$ in our settings, resulting \mbox{in fast processing time and small bandwidth overhead.   }

In Appendix~\ref{app:BSQ}, we prove that BSQ, without further assumptions, admits a worst-case \nmse of $ \frac{1}{n\cdot p\cdot\parentheses{2^b-1}^2}$ when using $b$ bits per quantized value. In particular, when $p$ and $b$ are constants, 
we get an \nmse of $O(1/n)$ with encoding and decoding times of $O(d)$ and $O(n\cdot d)$, respectively.

However, the linear dependence on $p$ means that the hidden constant in the $O(1/n)$ \nmse is often impractical. For example, if $p=2^{-5}$ and $b=1$, we need three bits per value on average: two for sending the exact values and their indices (assuming values are single precision floats and indices are 32-bit integers) and another for stochastically quantizing the remaining values using 1-bit stochastic quantization. \mbox{In turn, we get an \nmse bound of $\frac{1}{n \cdot 2^{-5}\cdot\parentheses{2^1-1}^2}=32/n$.}

In the following, we show that combining BSQ with our chosen random rotation preprocessing allows us to get an $O(1/n)$ \nmse with a much lower constant for small values of $p$. For example, a basic version of \alg{} with  $p=2^{-9}$ and $b=1$
can reach an $\nmse$ of $8.58/n$, a $3.72\times$ improvement despite using $2.66\times$ less bandwidth (i.e., $1.125$ bits per value instead of $3$).

\vspace*{-1.62mm}
\subsection{Distribution-aware unbiased quantization}\label{sec:alg:opt_sq}
\vspace*{-.5mm}
The first step towards our goal involves randomly rotating and scaling an input vector and then using BSQ to send values (rotated and scaled coordinates) outside the range $[-t_p, t_p]$ exactly. 
The values in the range $[-t_p, t_p]$ are sent using stochastic quantization, which ensures unbiasedness for any choice of quantization-values that cover that range.
Now we seek quantization-values that minimize the estimation variance and thereby the \nmse. 
We take advantage of the fact that,
after randomly rotating a vector $\overline x \in \mathbb{R}^d$ and scaling it by $\sqrt d / \norm{\overline x}_2$, the rotated and scaled coordinates approach the distribution of independent normal random variables $\mathcal N(0,1)$ as $d$ increases~\cite{vargaftik2021drive,EDEN}. Thus we choose to optimize the quantization-values for the normal distribution and later show that it yields a near-optimal quantization for the actual rotated coordinates.
That is, since we know both the distribution of the coordinates after the random rotation and scaling and we know the range of the values we are stochastically quantizing, we can design an unbiased quantization scheme that is optimized for this \mbox{specific distribution rather than using, e.g., the standard approach of uniformly sized intervals.}

Formally, for $b$ bits per quantized value and a BSQ parameter $p$, we find the set of quantization-values $\mathcal Q_{b,p}$ that minimizes the estimation variance of the random variable $Z \mid Z \in [-t_p, t_p]$ where $Z \sim \mathcal{N}(0,1)$, after stochastically quantizing it to a value in $\mathcal Q_{b,p}$ (i.e., the quantization is unbiased). 
Then, we \mbox{show how to use this precomputed set of quantization-values $\mathcal Q_{b,p}$ on any preprocessed vector. }

Consider parameters $p$ and $b$ and let $\mathcal X_b=\set{0,\ldots,2^b{-}1}$. Then, for a \textit{message} $x \in \mathcal X_b$, we denote by $S(z,x)$ the probability that the \textit{sender} quantizes a value $z \in [-t_p, t_p]$ to $R(x)$, the value that the \textit{receiver} associates with $x$. With these notations at hand, we solve the following optimization problem to find the set $\mathcal Q_{b,p}$ that minimizes the estimation variance (we are omitting the constant factor $1/\sqrt{2 \pi}$ in the normal distribution's pdf from the minimization as it does not affect the solution): 

\hspace{2mm}
\resizebox{.995\columnwidth}{!}{
$
\begin{array}{l}
\displaystyle{\minimize_{S,R}}  \displaystyle \int_{-t_p}^{t_p} \sum_{x\in \mathcal X_b} S(z,x) \cdot \parentheses{z-R(x)}^2 \cdot e^{\frac{-z^2}{2}} dz \qquad
\text{subject to}\\\medskip
{\small (\textit{\textcolor{gray}{Unbiasedness}})} \displaystyle\ \  \sum_{x \in \mathcal X_b} S(z,x) \cdot R(x) = z  \quad\forall\, z \in [-t_p, t_p]\\\medskip
{\small (\textit{\textcolor{gray}{Probability}})}\,\ \quad \displaystyle \sum_{x\in \mathcal X_b}S(z,x)=1 \quad \forall\, z \in [-t_p, t_p] \, , \qquad
S(z,x)\ge0 \quad \forall\, z \in [-t_p, t_p],\, x\in \mathcal X_b\\
\end{array}
$
}


Observe that $\mathcal Q_{b,p} = \set{R(x) \mid x\in\mathcal X_b}$ is the set of quantization-values that we are seeking. 

While there exist solutions to this problem \emph{excluding} the unbiasedness constraint (e.g., the Lloyd-Max Scalar Quantizer~\cite{lloyd1982least,max1960quantizing}), we are unaware of existing methods for solving the above problem analytically. 
Instead, we propose a discrete relaxation, allowing us to approach the problem with a \emph{solver}.\footnote{We use the Gekko~\cite{beal2018gekko} software package that provides a Python wrapper to the APMonitor~\cite{Hedengren2014} environment, running the solvers IPOPT~\cite{IPOPT} and APOPT~\cite{APOPT}.} 
To that end, we discretize the problem by approximating the truncated normal distribution using a finite set of $m$ \emph{quantiles}. Denote $\mathcal{I}_m = \set{0, \ldots, m-1}$ and let $Z \sim \mathcal{N}(0,1)$. Then, $\mathcal A_{p,m} = \set{A_{p,m}(i)\mid i\in\mathcal I_m}$, \mbox{where the quantile $\mathcal A_{p,m}(i)$ satisfies}

\smallskip
\hspace{19.91mm}
{$\qquad\qquad\Pr\brackets{Z \le \mathcal A_{p,m}(i)\ |\ Z\in[-t_p,t_p]} = \frac{i}{m-1}.$}
\smallskip

We find it convenient to denote $S'(i,x)= S(\mathcal A_{p,m}(i),x)$.
Accordingly, the discretized unbiased quantization problem is defined as (we omit the $1/m$ constant from the minimization as it does not affect the solution):
\vspace*{-1mm}

\hspace{5.4mm}
\resizebox{.92\columnwidth}{!}{
$
\begin{array}{l}
\displaystyle{\minimize_{S', R}}\ \ \ \ \displaystyle \sum_{i \in \mathcal{I}_m,x \in \mathcal{X}_b} S'(i,x) \cdot \parentheses{\mathcal A_{p,m}(i)-R(x)}^2\vspace*{-0mm} \qquad 
\text{subject to}\\\medskip
{\small (\textit{\textcolor{gray}{Unbiasedness}})} \displaystyle \sum_{x \in \mathcal{X}_b} S'(i,x) \cdot R(x) = \mathcal A_{p,m}(i)  \quad\forall\, i \in \mathcal{I}_m\\\medskip
{\small (\textit{\textcolor{gray}{Probability}})}\displaystyle\ \ \ \ \sum_{x \in \mathcal{X}_b}S'(i,x)=1\quad \forall\, i \in \mathcal{I}_m 
\,,\qquad S'(i,x)\ge0 \quad\forall\, i \in \mathcal{I}_m,\ x \in \mathcal{X}_b\\
\end{array}
$
}

\vspace*{-1.5mm}
The solution to this optimization problem yields the set of quantization-values $\mathcal Q_{b,p} = \set{R(x) \mid x\in\mathcal X_b}$ we are seeking. A value $z \in [-t_p, t_p]$ (not just the quantiles) is then stochastically quantized to one of the two nearest values in $\mathcal Q_{b,p}$. Such quantization is optimal for a fixed set \mbox{of quantization-values, so we do not need $S$ at this point.}

Unlike in vanilla BSQ (\cref{sec:alg:bsq}), in \alg{}, as implied by the optimization problem, the number of values that fall outside the range $[-t_p, t_p]$ may slightly deviate from $d\cdot p$ (and our guarantees are unaffected by this). This is because we precompute the optimal quantization-values set $\mathcal Q_{b,p}$ for a given $b$ and $p$ and set $t_p$ according to the $\mathcal{N}(0,1)$ distribution. In turn, this allows the \sender{}s to use $\mathcal Q_{b,p}$ when encoding rather than compute $t_p$ and then $\mathcal Q_{b,p}$ for each preprocessed vector separately.
This results in a near-optimal quantization for the actual rotated and scaled coordinates, in the sense that: (1) for large $d$ values, the distribution of the rotated and scaled coordinates converges to that of independent normal random variables; (2) for large $m$ values, 
\mbox
{the discrete problem converges to the continuous one.}
\vspace*{-2.5mm}
\subsection{Putting it all together}\label{sec:alg:alg}
\vspace*{-1.5mm}
 
The pseudo-code of \alg appears in Algorithm~\ref{alg:quickfl_initial_new}.
As mentioned, we use the uniform random rotation as a preprocessing stage done by the \senders. Crucially, similarly to~\citet{pmlr-v70-suresh17a}, and unlike in \citet{vargaftik2021drive,EDEN}, all \senders use the same rotation, which is a key ingredient in achieving fast decoding complexity. 

To compute this rotation (and its inverse by the \receiver), the parties rely on \emph{global shared randomness} as mentioned in \cref{sec:Preliminaries}. In practice, having shared randomness only requires the round's participants and the server to agree on a pseudo-random number generator seed, \mbox{which is standard practice.}

\vspace*{-2mm}
\paragraph{Clients.}
Each \sender $c$ uses global shared randomness to compute its rotated vector $T\parentheses{\overline x_c}$. Importantly, all \senders use the same rotation. 
As discussed, for large dimensions, the distribution of each entry in the rotated vector converges to $\mathcal N(0,\norm{\overline x_c}_2^2/d)$. Thus, $c$ normalizes it by $\sqrt d/\norm {\overline x_c}_2$ so the values of $\overline{Z}_c$ are approximately distributed as $\mathcal N(0,1)$ (line 1). (Note that we do {\em not} assume the values are {\em actually} normally distributed;   
this is {\em not} required for our algorithm or our analysis.)
Next, the client divides the preprocessed vector into large and small values (lines 2-4). 
The small values (i.e., whose absolute value is smaller than $t_p$) are stochastically quantized (i.e., in an unbiased manner) to values in the precomputed set $\mathcal Q_{b,p}$. We implement $\mathcal Q_{b,p}$ as an array where $\mathcal Q_{b,p}[x]$ stands for the $x$'th quantization-value; this allows us to transmit just the quantization-value indices over the network (line 5).
Finally, each client sends to the server the vector's norm $\norm {\overline x_c}_2$, the indices $\overline X_c$  of the quantization-values of $\overline V_c$ (i.e., the small values), and the exact large \mbox{values with their indices in $\overline{Z}_c$ (line 6).}

\vspace*{-2mm}
\paragraph{Server.} For each \sender $c$, the server uses $\overline X_c$ to look up the quantization-values $\widehat{\overline V}_{c}$ of the small coordinates (line 8) and constructs the estimated scaled rotated vector $\widehat{\overline Z}_{c}$ using $\widehat{\overline V}_{c}$ and the accurate information about the large coordinates $\overline U_c$ and their indices $\overline I_c$ (line 9). Then, the server computes the estimate $\widehat{\overline Z}_{\mathit{avg}}$ of the average rotated and scaled vector by averaging the reconstructed clients' scaled and rotated vectors and multiplying the results by the inverse scaling factor $\frac{\norm {\overline x_c}_2}{\sqrt d}$ (line 10). Finally, the server performs \emph{a single} inverse rotation using the global shared randomness {to obtain the estimate of the mean vector $\widehat {\overline x}_{\mathit{avg}}$ (line 11).}

\setlength{\textfloatsep}{2pt}
\begin{algorithm}[t]
\caption{~\alg}
\label{code:alg1}
\begin{algorithmic}
    \vspace{0.5mm} \State \hspace*{-4mm}\resizebox{.99995998\textwidth}{!}{\textbf{Input:} Bit budget $b$, BSQ parameter $p$, and their threshold $t_p$ and precomputed quantization-values $\mathcal Q_{b,p}$.}
    \vspace{-5.52mm}\\\hspace*{-4mm}\hrulefill
\end{algorithmic}    
\vspace{-5.252mm}
\begin{multicols}{2}
\begin{algorithmic}[1]
  \Statex \hspace*{-4mm}\textbf{\Sender{} $c$:}
    \smallskip
    \State $\overline Z_c \leftarrow \frac{\sqrt d}{\norm {\overline x_c}_2}\cdot T\parentheses{\overline x_c}$\textcolor{white}{$\big($}\smallskip
    \State $\overline U_c\leftarrow\set{ \overline{Z}_c[i] \,\big|\, \abs{\overline{Z}_c[i]} > t_p}$\smallskip 
    \State  $\overline I_c\leftarrow \set{ i \,\big|\, \abs{\overline{Z}_c[i]} > t_p}$\smallskip 
    \State $\overline V_c\leftarrow\set{ \overline{Z}_c[i] \,\big|\, \abs{\overline{Z}_c[i]} \le t_p}$\smallskip
    \State $\overline X_c \leftarrow$ Stochastically quantize $\overline V_c$ using $\mathcal Q_{b,p}$ 
    \State Send $\parentheses{\norm {\overline x_c}_2,\,\overline U_c,\,\overline I_c,\,\overline X_c}$ to \receiver 
\end{algorithmic}
\columnbreak

\begin{algorithmic}[1]
\setcounter{ALG@line}{6}
    \Statex \hspace*{-4mm}\textbf{\Receiver:}\smallskip
    \State  For all $c$: \smallskip
    \State  \hspace{4mm} $\widehat{\overline V}_{c} \leftarrow  \set{\mathcal Q_{b,p}[x] \text{ for $x$ in }\overline X_c}$ \smallskip
    \State  \hspace{4mm}  $\widehat{\overline Z}_{c} \leftarrow$ Merge $\widehat{\overline V}_{c}$ and $\parentheses{\overline U_c,\,\overline I_c}$\smallskip
    \State $\widehat{\overline Z}_{\mathit{avg}} \leftarrow \frac{1}{n}\cdot \sum_{c=0}^{n-1} \frac{\norm {\overline x_c}_2}{\sqrt d} \cdot \widehat{\overline Z}_{c}$\smallskip
    \State  $\widehat {\overline x}_{\mathit{avg}}  \leftarrow T^{-1}\parentheses{\widehat{\overline Z}_{\mathit{avg}}}$
\end{algorithmic}
\end{multicols}
\label{alg:quickfl_initial_new}
\vspace*{-1.2mm}
\end{algorithm}

In \cref{app:vnmse_proof}, we formally establish the following error guarantee for \alg{} (i.e., \cref{alg:quickfl_initial_new}). 

\begin{restatable}{theorem}{qflurrnmse}
\label{thm:qfl_urr_nmse}
 Let $Z \sim \mathcal N(0,1)$ and let $\widehat Z$ be its estimation by our distribution-aware unbiased quantization scheme. Then, for any number of clients $n$ and any set of $d$-dimentional input vectors $\set{\overline x_c \in \mathbb{R}^d\mid c\in\set{0,\ldots,n-1}}$, we have that \alg's \nmse respects 
 {\small
 $$\nmse = \frac{1}{n} \cdot \E\Big[\parentheses{Z {-} \widehat{Z}}^2 \Big]+ O\Big(\frac{1}{n} \cdot\sqrt{\frac{\log d}{d}}\Big).$$   
 }
\end{restatable}
\vspace{-0mm}
The theorem accounts for the cost of quantizing the actual rotated and scaled coordinates (which are not independent and follow a shifted-beta distribution) instead of independent and truncated normal variables. The difference manifests in the $O(1/n \cdot\sqrt{\log d/d}) = O(1/n)$ term; this quickly decays with the dimension and number of clients. 

As the theorem suggests, $\nmse \approx \frac{1}{n} \cdot \E[(Z {-} \widehat{Z})^2]$ for \alg{} in settings of interest. 
Moreover,

\hspace{28mm}
\resizebox{.602\columnwidth}{!}{
$
\begin{array}{l}
   \! \! \! \! \E\Big[{\parentheses{Z {-} \widehat{Z}}^2}\Big] = 
    \E\Big[\parentheses{Z {-} \widehat{Z}}^2 \big |\ Z \in [-t_p, t_p]\Big] \cdot \Pr[ Z \in [-t_p, t_p]] \\     + \mbox{    }\E\Big[\parentheses{Z {-} \widehat{Z}}^2 \big |\  Z \not\in [-t_p, t_p]\Big]\cdot \Pr[ Z \not\in [-t_p, t_p]]\ \ ,   
\end{array}
$
}

where the first summand is exactly the quantization error of our distribution-aware unbiased BSQ, and the second summand is $0$ as such values are sent exactly.
This means that for any $b$ and $p$, we can exactly compute $\E[(Z {-} \widehat{Z})^2]$ given the solver's output (i.e., the precomputed quantization \mbox{values). For example, it is $\approx 8.58$ for $b=1$ and $p=2^{-9}$. }
Another important corollary of \cref{thm:qfl_urr_nmse} is that the convergence speed with QUIC-FL matches the vanilla SGD since its estimates are unbiased and with an $O(1/n)$ \nmse (e.g., see Remark 5 in \citet{karimireddy2019error}).

\vspace{-1mm}
\subsection{Optimizations}\label{sec:alg:scary_stuff_no_one_will_understand}
\vspace{-2mm}

We introduce two optimizations for \alg{}: we further reduce \nmse with client-specific shared randomness and then accelerate the processing time via the randomized Hadamard transform.
\vspace{-1mm}

\paragraph{\alg{} with client-specific shared randomness.}

Past works (e.g.,~\citet{ben2020send,chen2020breaking,roberts1962picture}) on optimizing the quantization-bandwidth tradeoff show the benefit of using \emph{shared randomness} to reduce the quantization error. Here, we show how to leverage this (client-specific) \mbox{shared randomness to design near-optimal quantization of the rotated and scaled vector. }

To that end, in Appendix~\ref{app:hell}, we first extend our optimization problem to allow client-specific shared randomness and then derive the related discretized problem. Importantly, we also discretize the client-specific shared randomness where each client, for each rotated and quantized coordinate, uses a shared random $\ell$-bit value $H\sim \mathcal U[\mathcal H_l]$  where $\mathcal H_\ell = \set{0, \dots, 2^\ell -1}$. 

The resulting optimization problem is given as follows (additions are highlighted in red):

\hspace{0mm}
\resizebox{0.99\columnwidth}{!}{
$
\begin{array}{l}
\displaystyle{\minimize_{S',R}}  \displaystyle \sum_{\substack{\textcolor{red}{h\in\mathcal H_\ell}\\i\in\mathcal I_m\\x\in \mathcal X_b}} S'(\textcolor{red}{h,}\ i,x) \cdot \parentheses{\mathcal A_{p,m}(i)-R(\textcolor{red}{h,}\ x)}^2\medskip \qquad
\text{subject to}\\
{\small (\textit{\textcolor{gray}{Unbiasedness}})~} \displaystyle\textcolor{red}{\frac{1}{2^{\ell}}\ \cdot}\sum_{\substack{\textcolor{red}{h\in\mathcal H_\ell}\\x \in \mathcal X_b}} S'(\textcolor{red}{h,}\ i,x) \cdot R(\textcolor{red}{h,}\ x) = \mathcal A_{p,m}(i) \ \ \hfill \forall\, i \in \mathcal I_m\medskip\medskip\\\medskip\medskip
{\small (\textit{\textcolor{gray}{Probability}})}\,\,\,\,\,\displaystyle \sum_{x\in \mathcal X_b}S'(\textcolor{red}{h,}\ i,x)=1
\hfill\forall\,\textcolor{red}{h \in \mathcal H_\ell,}\,\, i \in \mathcal I_m
\, , \quad S'(\textcolor{red}{h,}\ i,x)\ge0
\quad \forall\, \textcolor{red}{h \in \mathcal H_\ell,}\,\,i \in \mathcal I_m, \,\, x \in \mathcal X_b
\end{array}
$
}\vspace*{-2mm}
Here $S'(h,i,x) = S(h,\mathcal A_{p,m}(i),x)$ represents the probability that the sender sends the message $x\in\mathcal X_b$ given the shared randomness value $h$ for the input value $ \mathcal A_{p,m}(i)$.   
 Similarly, $R(h,x)$ is the value the receiver associates with the message $x$ when the shared randomness is $h$. 
 We explain how to use $R(h,x)$ to determine the appropriate message for the sender on a general input $z$, along with further details, in Appendix~\ref{app:hell}. We note that~\cref{thm:qfl_urr_nmse} trivially applies to \alg with client-specific shared randomness as this only lowers the quantization's expected squared error, i.e., $\E[({Z - \widehat{Z})}^2 ]$, and thus the resulting \nmse.

Here, we provide an example based on the solver's solution for the case of using a single shared random bit (i.e., $H\sim \mathcal U[\mathcal H_1]$), a single-bit message ($b=1$), and $p=2^{-9}$ ($t_p\approx 3.097$);
We can then use the following algorithm, where $X$ is the sent message and $\alpha=0.7975, \beta=5.397$ are constants:

\vspace*{-4mm}
{\small
\begin{equation*}
X = 
\begin{cases}
1 & \mbox{if $H = 0$ and $Z\ge0$}\\
0 & \mbox{if $H = 1$ and $Z<0$}\\
\mathit{Bernoulli}(\frac{2Z}{\alpha+\beta})& \mbox{If $H=1$ and $Z\ge0$}\\
1-\mathit{Bernoulli}(\frac{-2Z}{\alpha+\beta})& \mbox{If $H=0$ and $Z<0$}\\
\end{cases}
%
%
\qquad\qquad\qquad
\widehat Z = 
\begin{cases}
-\beta & \mbox{if $H=X=0$}\\
-\alpha & \mbox{if $H = 1$ and $X=0$}\\
\alpha  & \mbox{If $H=0$ and $X=1$}\\
\beta  & \mbox{If $H=X=1$}\\
\end{cases}\ .
\end{equation*}
}
{For example, consider $Z=1$, and recall that $H=0$ w.p. $1/2$ and $H=1$ otherwise. Then:}
\begin{itemize}[align=left, leftmargin=0mm, labelindent=0\parindent, listparindent=0\parindent, labelwidth=0mm,itemindent=!,itemsep=1pt,parsep=0pt,topsep=1pt]
    \item If $H=0$, we have $X=1$ and thus $\widehat Z = \alpha$.
    \item If $H=1$, then $X=1$ w.p. $\frac{2}{\alpha+\beta}$ and we get $\widehat Z = \beta$. Otherwise (if $X=0$), we get $\widehat Z = -\alpha$.
\end{itemize} 
Indeed, we have that the estimate is unbiased since:
{\small
\begin{equation*}
\resizebox{.716\hsize}{!}{%
$
\mathbb E[\widehat Z \mid Z=1] = \frac{1}{2}\cdot \alpha + \frac{1}{2}\cdot\parentheses{ \frac{2}{\alpha+\beta}\cdot \beta + \frac{\alpha+\beta-2}{\alpha+\beta}\cdot (-\alpha)}
= 1.
$
}
\end{equation*}
}

\vspace{-2mm}
\mbox{We next calculate the expected squared error (by symmetry, we integrate over positive $z$):}
%
{\small
\begin{equation*}
\resizebox{1.01\hsize}{!}{%
$
\mathbb E\brackets{(Z-\widehat Z)^2} = \sqrt{\frac{2}{\pi}}\bigg(\int_0^{t_p} \frac{1}{2}\cdot\Big((z-\alpha)^2 +\quad \frac{2z}{\alpha+\beta}\cdot (z-\beta)^2 + \frac{\alpha+\beta-2z}{\alpha+\beta}\cdot (z+\alpha)^2\Big)\cdot e^{-z^2/2}dz\bigg)\approx 3.29.
$
}
\end{equation*}
}

\vspace{-2mm}
Observe that it is significantly lower than the 8.58 quantization error obtained without shared randomness. \mbox{As we illustrate (\cref{fig:sensitivity}), the error further decreases when using more shared random bits.}

\vspace*{-2mm}
\paragraph{Accelerating \alg{} with RHT.}

Similarly to previous algorithms that use random rotations as a preprocessing state (e.g.,~\citet{pmlr-v70-suresh17a,vargaftik2021drive,EDEN}) we propose to use the Randomized Hadamard Transform (RHT) \cite{ailon2009fast} instead of uniform random rotations. Although RHT does not induce a uniform distribution on the sphere, it is considerably more efficient to compute, and, under mild assumptions, the resulting distribution is close to that of a uniform random rotation~\cite{vargaftik2021drive}. Nevertheless, we are interested in establishing how using RHT instead of a uniform random rotation affects the formal guarantees of \alg{}. 
 
As shown in Appendix~\ref{app:hadamard}, \alg with RHT has the same asymptotic guarantee as with random rotations, albeit with a larger constant (constant factor increases in the fraction of exactly sent values and \nmse). We note that these guarantees are still stronger than those of DRIVE~\cite{vargaftik2021drive} and EDEN~\cite{EDEN}, which only prove RHT bounds for vectors whose coordinates are sampled \mbox{i.i.d. from a distribution with finite moments, and are not applicable to adversarial vectors.}

For example, when $p=2^{-9}$ and we use $\ell=4$ shared random bits per quantized coordinate, our analysis shows that the \nmse for $b=1,2,3,4$ is bounded by $4.831/n, 0.692/n,\allowbreak 0.131/n, 0.0272/n$, accordingly, and that the expected number of coordinates outside $[-t_p,t_p]$ is bounded by $3.2\cdot p\cdot d \approx 0.006\cdot d$. 
We note that this result does not have the $O\big(1/n \cdot\sqrt{\log d/d}\big)$ additive \nmse term. The reason is that we directly analyze the error for the Hadamard-rotated coordinates (whereas Theorem~\ref{thm:qfl_urr_nmse} relies on analyzing the error in quantizing normal variables and factoring in the difference in distributions).
In particular, we get that for $p=2^{-9}, b\in\set{1,2,3}$, running \alg with Hadamard and $(b+1+2.2\cdot p)\approx b+1.0043$ bits per coordinate has lower \nmse than $b$-bits \alg with uniform random rotation. That is, one can compensate for the increased error caused by using RHT by adding one bit per coordinate. In practice, as shown in the evaluation, the actual performance is (as one might expect) actually close to the theoretical results for uniform random rotations; 
\mbox{improving the bounds is left as future work.}

\begin{figure*}[t]
\centering
\includegraphics[width=\linewidth]{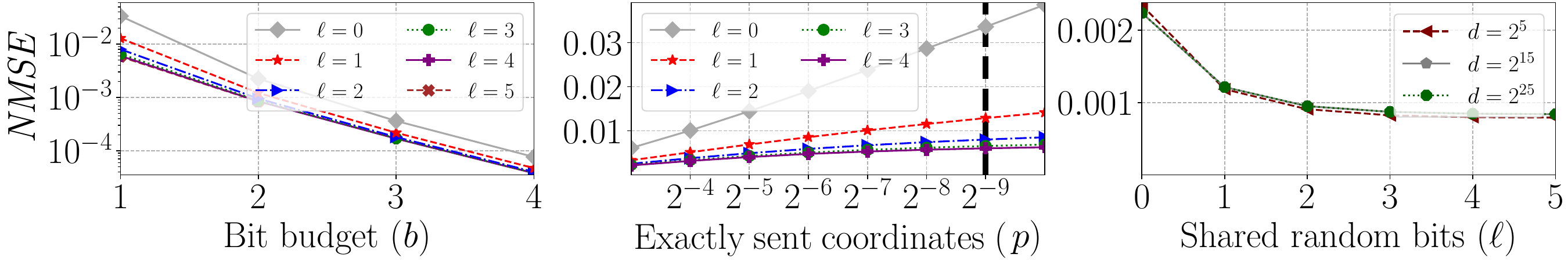}
\vspace*{-7mm}
\caption{The \nmse of \alg (with $n=256$ clients) as a function of the bit budget $b$, fraction $p$, and shared random bits $\ell$. In the leftmost figure, $p=2^{-9}$, while the other two use $b=4$.\ran{what is $p$ in the leftmost figure and $b$ in the other two}}\label{fig:sensitivity}
\vspace*{-4mm}
\end{figure*}

\begin{figure*}[]

\centering
\hspace*{-6mm}\includegraphics[width=0.195\linewidth]{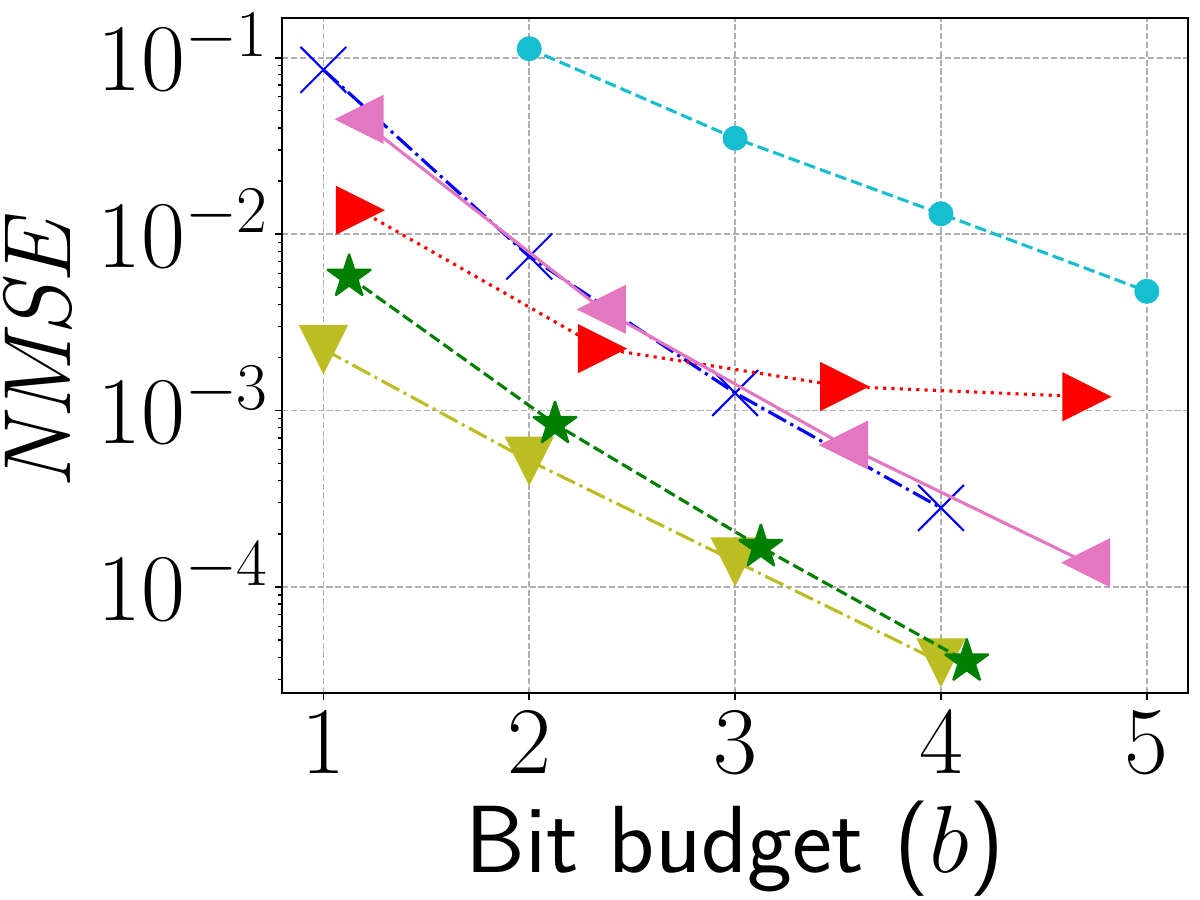}
\hspace*{-0.5mm}\includegraphics[width=0.195\linewidth]{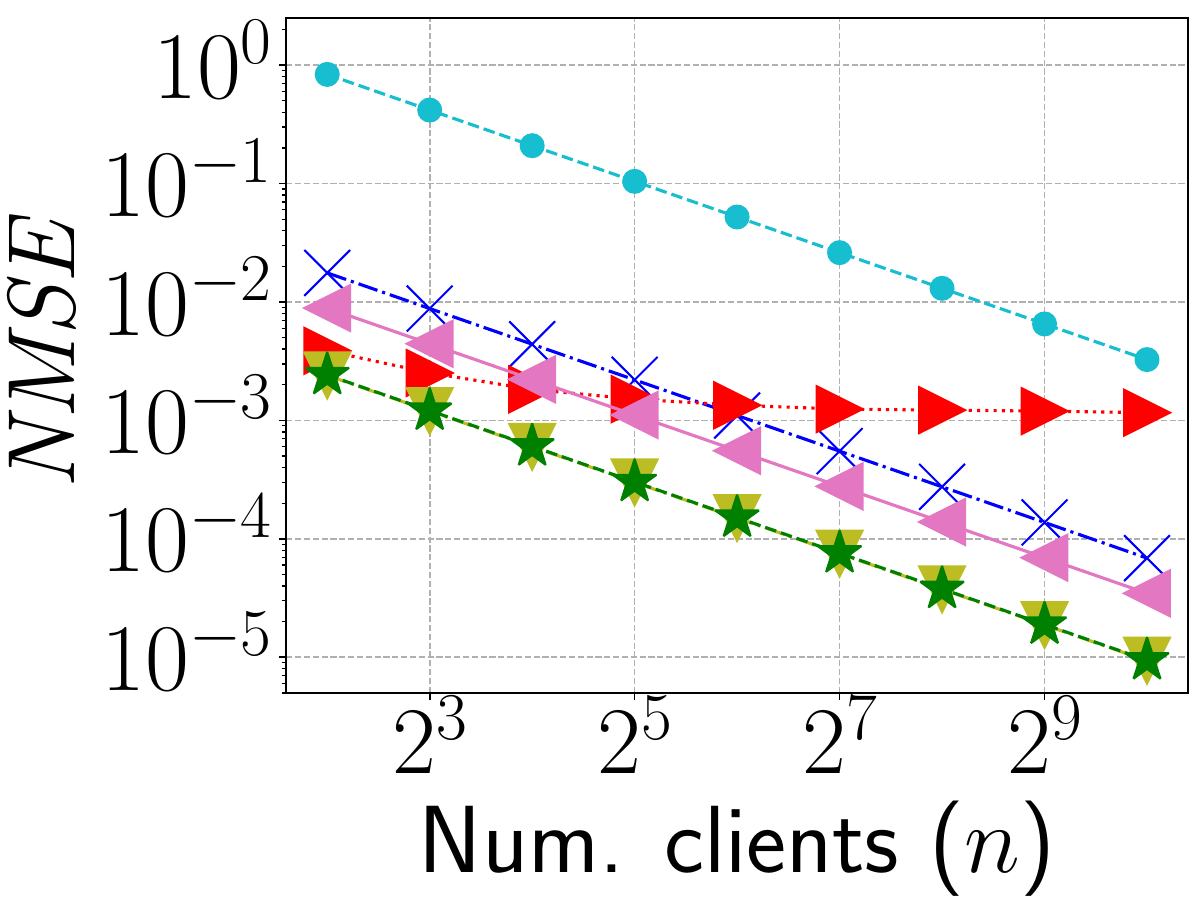}
\hspace*{-0.5mm}\includegraphics[width=0.195\linewidth]{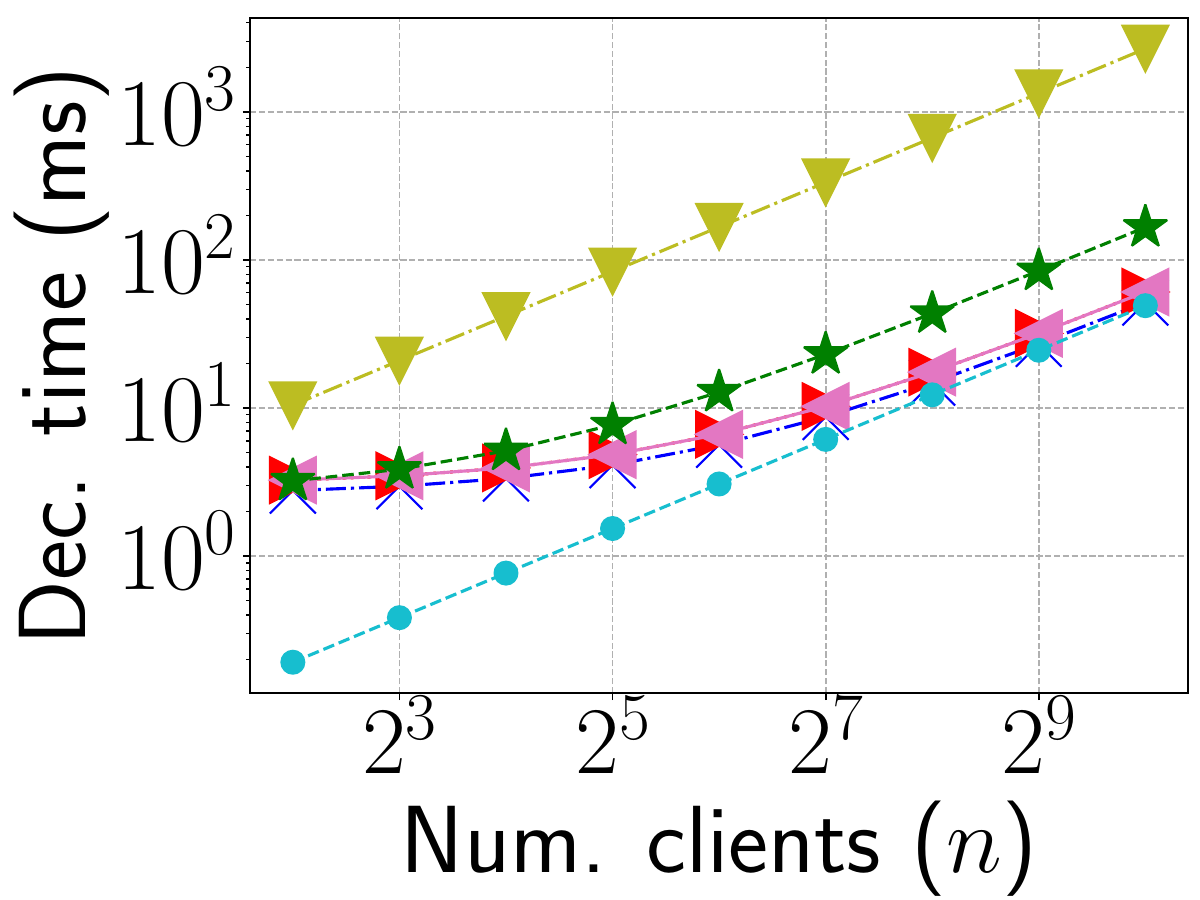}
\hspace*{-0.5mm}\includegraphics[width=0.195\linewidth]{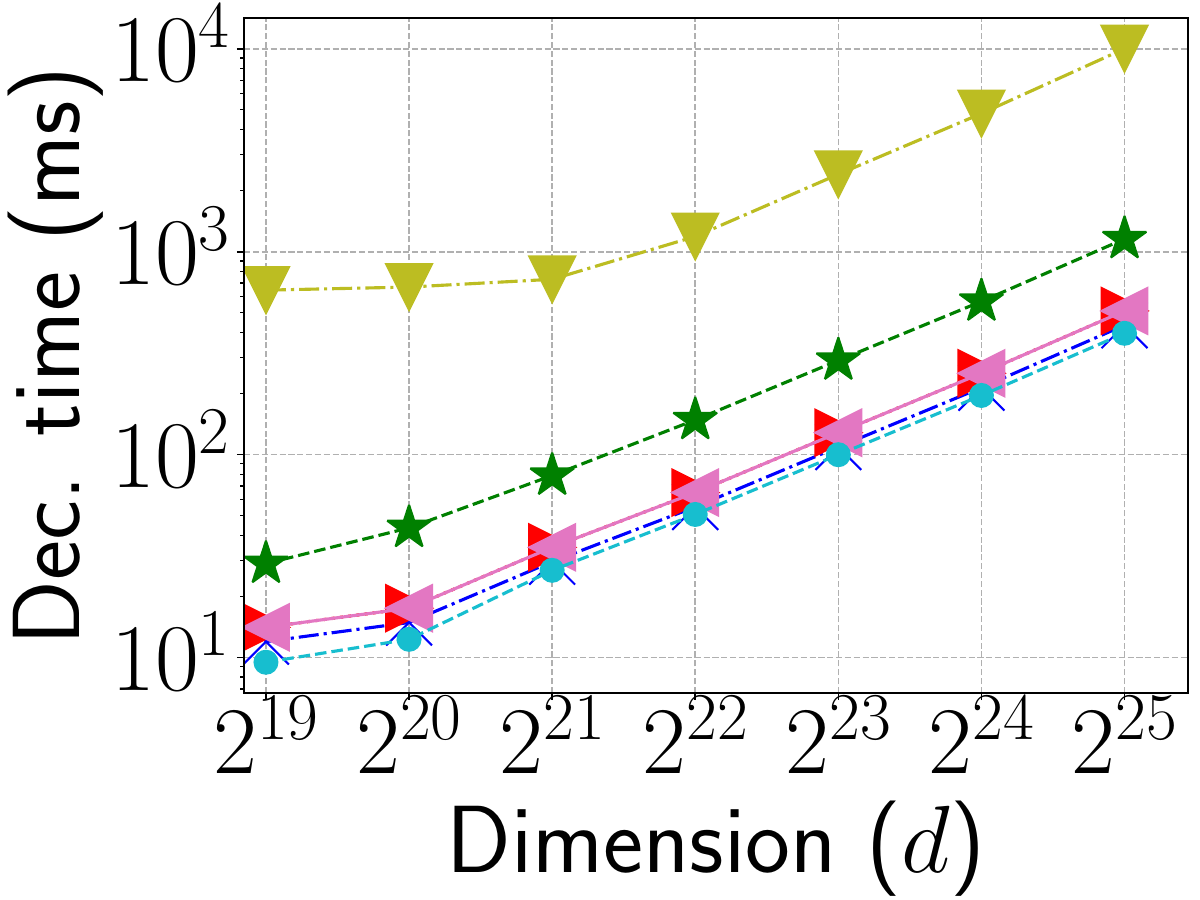}
\hspace*{-0.5mm}\includegraphics[width=0.195\linewidth]{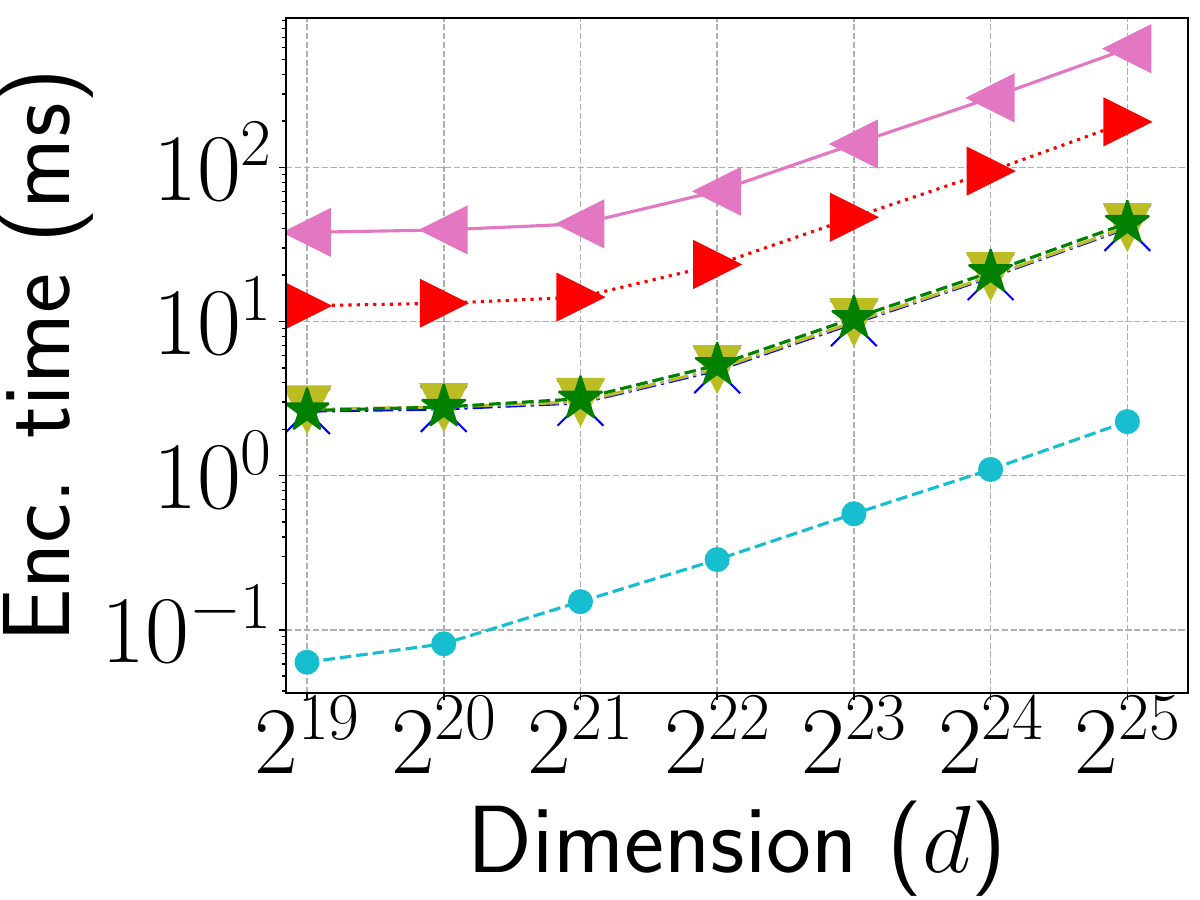}
\includegraphics[width=0.795\linewidth]{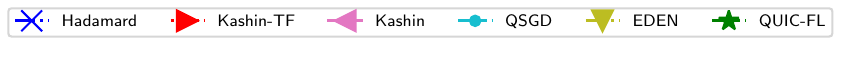}\\
\vspace*{-4mm}
\caption{Comparison to alternatives with $n$ clients that have the same $LogNormal(0,1)$ input vector. The default values are $n=256$ clients, $b=4$ bit budget, and $d=2^{20}$ dimensions. 
}\label{fig:NMSE}
\vspace*{2mm}
\end{figure*}

\looseness=-1
Finally, \cref{tbl:asymptotics} summarizes the theoretical guarantees of \alg{} in comparison to state-of-the-art DME techniques. The encoding complexity of \alg{} is dominated by RHT and is done in $O(d\cdot\log d)$ time. The decoding of \alg{} only requires the addition of all estimated rotated clients' vectors and \emph{a single} inverse RHT transform resulting in $O(n \cdot d + d \cdot \log d)$ time. {As mentioned, the \nmse with RHT remains $O(1/n)$. }
Observe that \alg{} has an \emph{asymptotic} speed improvement either at the clients or the server among the techniques that achieve $O(1/n)$ \nmse.


{
\vspace*{-3mm}
\section{Evaluation}\label{sec:expr:next-word}
\vspace*{-2mm}
\label{sec:evaluation}
In this section, we evaluate the fully-fledged version of \alg that leverages RHT and client-specific shared randomness, as given in~\cref{app:hell,alg:final}.

\T{Parameter selection.}
We experiment with how the different parameters (number of quantiles $m$, the fraction of coordinates sent exactly $p$, the number of shared random bits $\ell$, etc.) affect the performance of our algorithm. As shown in Figure~\ref{fig:sensitivity}, introducing shared randomness significantly decreases the \nmse compared with \cref{alg:quickfl_initial_new} (i.e., $\ell=0$). We note that these results are essentially independent of the input data (because of the RHT).
Additionally, the benefit from adding each additional shared random bit diminishes, and the gain beyond $\ell=4$ is negligible, especially for large $b$. Accordingly, we hereafter use $\ell=6$ for $b=1$, $\ell=5$ for $b=2$, and $\ell=4$ for $b\in\set{3,4}$. With respect to $p$, we determined $1/512$ as a good balance between the \nmse and bandwidth overhead for accurately sent values and their indices.
}

\looseness=-1
\T{Comparison to state-of-the-art DME techniques.}
Next, we compare the performance of \alg to the baseline algorithms in terms of \nmse, encoding speed, and decoding speed, using an NVIDIA 3080 RTX GPU machine with 32GB RAM and i7-10700K CPU @ 3.80GHz.
Specifically, we compare with inputs where each
coordinate is independently $LogNormal(0,1)$~\cite{chmiel2020neural}.
Hadamard~\cite{pmlr-v70-suresh17a}, Kashin's representation~\cite{caldas2018expanding,safaryan2020uncertainty}, QSGD~\cite{NIPS2017_6c340f25}, and EDEN~\cite{EDEN}.
We evaluate two variants of Kashin's representation: (1) The TensorFlow (TF) implementation~\cite{tensorflowfedkashincode} that, by default, limits the decomposition to three iterations, and (2) the theoretical algorithm that requires $O(\log (n\cdot d))$ iterations.
For this experiment, the coordinates are 
As shown in Figure~\ref{fig:NMSE}, \alg has significantly faster decoding than EDEN (as previously conveyed in \cref{fig:intro_example}), the only alternative with competitive \nmse. 

\alg is also significantly more accurate than all other approaches. We observe that the default TF configuration of Kashin's representation suffers from a bias, and therefore its \nmse is not $O(1/n)$. In contrast, the theoretical algorithm is unbiased but has an asymptotically slower encoding time. We observed similar trends for different $n,b$, and $d$ values. We consider the algorithms' bandwidth over all coordinates (i.e., with $b+\frac{64}{512}$ bits for \alg, namely a float and a 32-bit index for each accurately sent entry).
Overall, the empirical measurements fall in line with the bounds in Table~\ref{tbl:asymptotics}.

\begin{figure*}[h!]
\centering
\includegraphics[clip,width=0.95\linewidth]{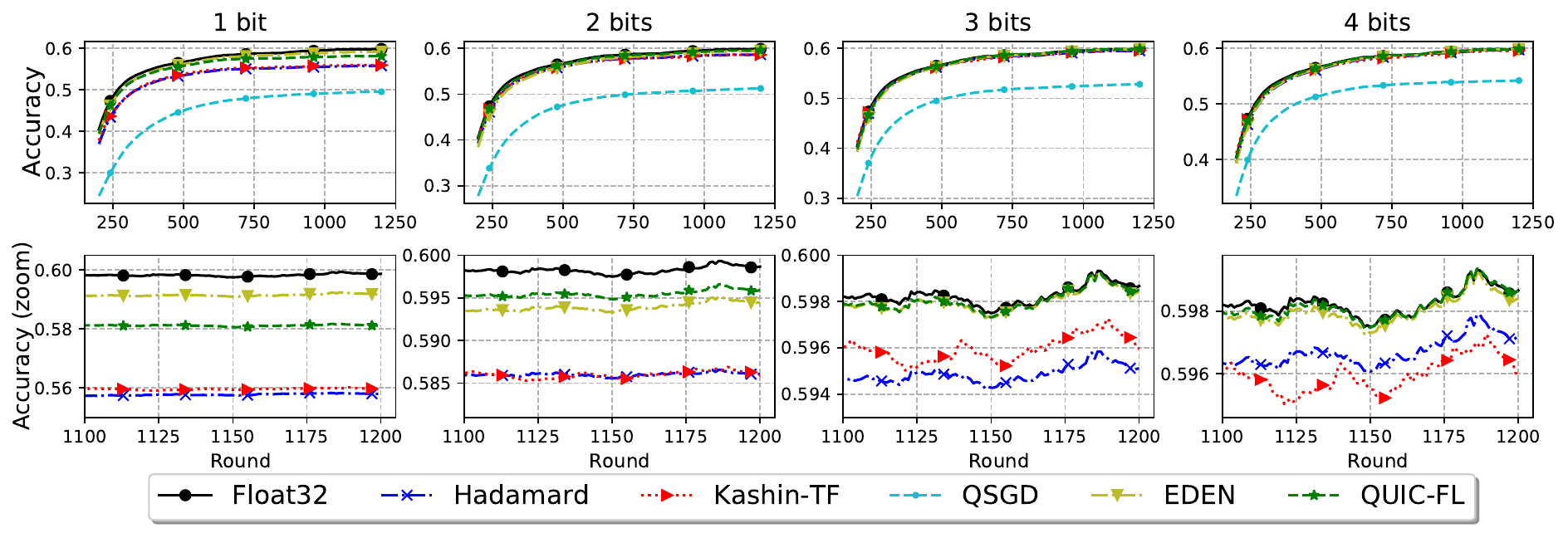}
\vspace*{-1mm}
\caption{\emph{FedAvg} over the Shakespeare next-word prediction task at various bit budgets (rows). We report training accuracy per round with a rolling mean of 200 rounds. 
}
\label{fig:shakespeare}
\end{figure*}


\T{Federated Learning Experiments.} 
\looseness=-1
We evaluate \alg over the Shakespeare next-word prediction task \cite{shakespeare, mcmahan2017communication} using an LSTM recurrent model. It was first suggested in \cite{mcmahan2017communication} to naturally simulate a realistic heterogeneous federated learning setting. We run \emph{FedAvg}~\cite{mcmahan2017communication} with the Adam server optimizer~\cite{KingmaB14} and sample $n=10$ clients per round. We use the setup from the federated learning benchmark of \cite{reddi2021adaptive}, restated for convenience in Appendix~\ref{app:expr-details}.
Figure~\ref{fig:shakespeare} shows how \alg 
is competitive with the asymptotically slower EDEN and markedly more accurate than other alternatives.


Due to space limits, experiments for image classification (\cref{app:imagecfl}), a framework that uses DME \mbox{as a building block (\cref{app:ef21_topk_qfl}), and power iteration (\cref{app:subsec:pi}), appear in the appendix.}




\vspace*{-2mm}
\section{Related Works}

In \cref{sec:intro}, we gave an overview of most related compression (DME) methods. Here we focus on specific sub-topics. We further review other techniques in \cref{app:extended_RW}.

\T{Bounded support quantization.} Previous works on compression in federated learning observed considered bounding the range of the updates. They suggest ad-hoc mitigations, such as clipping \cite{Zhang2020Why, wen2017terngrad, pmlr-v162-zhang22b, NEURIPS2021_ab9ebd57}, preconditioning \cite{pmlr-v70-suresh17a, caldas2018expanding}, and bucketing \cite{NIPS2017_6c340f25}. On the other hand, methods such as Top-$k$ \cite{NEURIPS2018_b440509a, NEURIPS2020_a851bd0d} demonstrate that considering the largest coordinates is advantageous. \citet{horvath2021a} provides convergence guarantees from combining biased and unbiased compressed estimators.  
BSQ similarly tries to benefit by sending the largest transformed coordinates exactly while sending the rest via unbiased compression.

\T{Distribution-aware quantization.}
Quantization over a distribution, and over a Gaussian source in particular, has been studied for almost a century (for a comprehensive overview, we refer to \citet{720541}). Nevertheless, to our knowledge, such research has not focused on the unbiasedness constraint. The only comparable methods that we are aware of are based on stochastic quantization and introduce an error that increases with the vector's dimension. There are additional unbiased methods that use shared randomness  (e.g., \citet{1057702, ben2020send}), but again, we are unaware of any work that directly optimizes quantization for a distribution with an unbiasedness constraint.
As previously mentioned, perhaps the closest to our approach is the  Lloyd-Max Scalar Quantizer~\cite{lloyd1982least,max1960quantizing}, which optimizes the mean squared error without unbiasedness constraints. Interestingly, there are many generalizations to Lloyd-Max, such as vector quantization \cite{1094577} methods and lattice quantization \cite{1056067}. In future work, we plan to investigate these approaches and extend our distribution-aware \mbox{unbiasedness quantization framework accordingly.}

\AP{higher dimension a-la VQ + non-uniform discretiztion methods}

\bibliography{example_paper}
\bibliographystyle{icml2023}

\newpage
\appendix
\onecolumn

\section{Extended Related Work}\label{app:extended_RW}

This paper focused on the Distributed Mean Estimation (DME) problem where \senders send lossily compressed vectors to a centralized \receiver for averaging.
While this problem is worthy of study on its own merits, we are particularly interested in applications to federated learning, where there are many variations and practical considerations, which have led to many alternative compression \mbox{methods to be considered. }

We note that in essence, \alg is a compression scheme. However, unlike previous DME approaches such as \cite{pmlr-v70-suresh17a,vargaftik2021drive,EDEN}, it brings benefits \emph{only} in a distributed setting with multiple clients, distinguishing it from standard vector quantization methods.

\paragraph{Frameworks that use DME as a building block.}
In addition to EF21~\cite{richtarik2021ef21} and MARINA~\cite{pmlr-v139-gorbunov21a,szlendak2022permutation} which are discussed in detail below, there are additional frameworks that leverage DME as a building block. For example, EF-BV~\cite{condat2022ef}, Qsparse-local-SGD~\cite{basu2019qsparse}, 3PC~\cite{richtarik20223pc}, 
CompressedScaffnew~\cite{condat2022provably},
MURANA~\cite{condat2022murana}, and DIANA~\cite{horvath2023stochastic} accelerate the convergence of non-convex learning tasks via variance reduction, control variates, and compression. These approaches are orthogonal and \mbox{can benefit from better DME techniques such as \alg.}

\paragraph{Entropy encoding.}
When the encoding and decoding time is less important, some previous  approaches have suggested using an entropy encoding such as Huffman or arithmetic encoding to improve the accuracy (e.g.,~\citet{NIPS2017_6c340f25,pmlr-v70-suresh17a,EDEN,dorfman2023}). 
Intuitively, such encodings allow us to losslessly compress the lossily compressed vector to reduce its representation size, thereby allowing less aggressive quantization.
However, we are unaware of available GPU-friendly entropy encoding implementation and thus such methods incur a significant time overhead.

\paragraph{Client-side memory.}
Critically, for the basic DME problem, the assumption is that this is a one-shot process where the goal is to optimize the accuracy without relying on client-side memory. This model naturally fits cross-device federated learning, where different clients are sampled in each round.
We focused on unbiased compression, which is standard in prior works~\cite{pmlr-v70-suresh17a,konevcny2018randomized,vargaftik2021drive,davies2021new,mitchell2022optimizing}.
However, if the compression error is low enough, and under some assumptions, SGD can be proven to \mbox{converge even with biased compression~\cite{abs-2002-12410}.}

\paragraph{Error feedback.}
In other settings, such as distributed learning or cross-silo federated learning, we may assume that \senders are persistent and have a memory that keeps state between rounds.
A prominent option to leverage such a state is to use Error Feedback (EF).
In EF, \senders can track the compression error and add it to the vector computed in the consecutive round. This scheme is often shown to recover the model's convergence rate and resulting accuracy~\cite{seide20141,alistarh2018convergence,richtarik2021ef21,karimireddy2019error} and enables biased {compressors such as Top-$k$~\cite{NEURIPS2018_b440509a} and SignSGD~\cite{bernstein2018signsgd}.}
We compare with the state of the art technique, EF21~\cite{richtarik2021ef21}, in addition to showing how it can be used in conjunction with \alg to facilitate further improvement in Appendix~\ref{app:eval}.

\paragraph{Gradient differences.}
An orthogonal proposal that works with persistent clients, which is also applicable with \alg, is to encode the difference between the current vector and the previous one instead of directly compressing the vector~\cite{mishchenko2019distributed, pmlr-v139-gorbunov21a}. Broadly speaking, this allows a compression error proportional to the L2 norm of the difference and not the vector and can decrease the error if consecutive vectors are similar to each other.


\paragraph{In-network aggregation.}
When running distributed learning in cluster settings, recent works show how in-network aggregation can accelerate the learning process~\cite{sapio2021scaling,lao2021atp,segal2021soar,li2023thc}.
IntSGD~\cite{mishchenko2022intsgd} is another compression scheme that allows one to aggregate the compressed integer vectors in the network. However, their solution may require sending $14$ bits per coordinate while we consider $1-5$ bits per coordinate in QUIC-FL. %
Intuitively, switches are designed to move data at high speeds, and recent advances in switch programmability enable them to easily perform simple {aggregation operations like summation while processing the data.} \mbox{Extending \alg to allow efficient in-network aggregation is left as future work.}

\paragraph{Sparsification.}
Another line of work focuses on sparsifying the vectors before compressing them~\cite{konecy2017federated,aji2017sparse,konevcny2018randomized,wangni2018gradient,stich2018sparsified,fei2021efficient,EDEN}. Intuitively, in some learning settings, many of the coordinates are small, and we can improve the accuracy to bandwidth tradeoff by removing all small coordinates prior to compression. Another form of sparsification is random sampling, which allows us to avoid sending the coordinate indices~\cite{konecy2017federated,EDEN}. We note that combining such approaches with \alg is straightforward, as we can use \alg to compress just the non-zero entries of the sparsified vectors.

\paragraph{Deep gradient compression.}
By combining techniques like warm-up training, vector clipping, momentum factor masking, momentum correction, and deep vector compression, \cite{LinHM0D18} {reports savings of two orders of magnitude in the bandwidth required for distributed learning. }

\paragraph{Shared randomness.}
As shown in~\cite{ben2020send}, shared randomness can reduce the worst-case error of quantizing a single $[0,1]$ value both in biased and unbiased settings. However, applying this approach directly to the vector's entries results in $O(d/n)$ \nmse for any $b=O(1)$.
Another promising orthogonal approach is to leverage shared randomness to push the clients' compression to yield errors in opposite directions, thus making them cancel out and lowering the overall NMSE~\cite{suresh2022correlated,szlendak2022permutation}.

\ifarxiv
\paragraph{Network applications.}
Compression is also fundamental in network telemetry~\cite{basat2020faster,basat2021salsa}, where it allows devices to communicate fewer bits while ensuring an accurate network-wide view at the controller~\cite{ben2020pint,basat2022memento}.

\paragraph{Network efficiency.}
Another, often orthogonal, approach to accelerate DNN training and improve its efficiency is to use emerging high-bandwidth and energy-efficient technologies such as optical switching \cite{khani2021sip} with matching low-latency scheduling techniques (e.g., \cite{goel,huang2016sunflow,solstice,vargaftik2016composite,livshits2018lumos,eclipse,vargaftik2020c}).
\else
\fi

\paragraph{Non-uniform quantization.}
The \alg algorithm, based on the output of the solver (see \S\ref{sec:alg}), uses \emph{non-uniform} quantization, i.e., has quantization levels that are not uniformly spaced. Indeed, recent works observed that non-uniform quantization improves the estimation accuracy and accelerates the learning convergence~\cite{ramezani2019nuqsgd,EDEN}.

\paragraph{Correlations.}
Some techniques further reduce the error by leveraging potential correlations between coordinates~\cite{mitchell2022optimizing} or client vectors~\cite{davies2021new}; it is unclear how to combine these with \alg \mbox{and we leave this for future work.}

\paragraph{Privacy concerns} Several works optimize the communication-accuracy tradeoff while also considering the privacy of clients' data. For example the authors of~\cite{chen2020breaking} optimize the triple communication-accuracy-privacy tradeoff, while~\cite{gandikota2021vqsgd}
addresses the harder problem of compressing the gradients while maintaining differential privacy. Their results can be split into two groups: (1) algorithms that require $O(\log d)$ bits per coordinate to reach the $O(1/n)$ NMSE, and (2) an algorithm that needs $O_\epsilon(1)$ bits per coordinate (which hides functions of $\epsilon$) to reach an NMSE of $\frac{1}{n\cdot(1-\epsilon)}$. In particular, the vNMSE of this approach is always larger than that of \alg, even for $b=1$.

\paragraph{Spherical compression.}
Spherical compression (SC)~\cite{albasyoni2020optimal} is a highly accurate biased quantization method that draws random points on a unit sphere until one is $\epsilon$-close to the vector's direction; it then sends just the number of points needed and the server uses the same pseudo-random number generator seed to compute the estimate.
The algorithm runs in time $O(d / \mathfrak p)$, where $\mathfrak p$ is the probability that a sampled point is $\epsilon$-close to the input and satisfies $\mathfrak p=\tfrac12\,F_{(d-1)/2,\,1/2}(\alpha)$, where $F$ is the CDF of the Beta distribution and $\alpha$ is desired the \vnmse bound.
Evaluating this expression shows that is excessively large when $d$ is not very small. For example, for $d=100$, they would require over $10^{33}$  samples on average (while we consider $d$ in the millions). 
More generally, $1/\mathfrak p \ge (1/\alpha)^{d/2}$, thus the encoding and decoding complexities are exponential. This is implied by the lower bound of~\cite{safaryan2020uncertainty}. Finally, we note that QUIC-FL is unbiased while the SC algorithm is biased (and thus, its \nmse does not decrease linearly in $n$).



\paragraph{Sparse dithering.}
Sparse dithering is a compression method that is shown to be near-optimal in the sense that it requires at most constant factor more bandwidth than the lower bound for the same error rate. We compare with it in Appendix~\ref{app:sd}.

\paragraph{Natural compression}
Natural Compression and Dithering~\cite{horvoth2022natural} are schemes optimized for processing speed by taking into consideration the representation of floating point values when designing the compression. However, In order to get constant \vnmse, they seem to require $O(d\log d)$ bits  compared with $O(d)$ bits in QUIC-FL and their \vnmse is lower bounded by $1/8$, while QUIC-FL achieves a \vnmse of $\approx0.0444,\approx0.00982$ with $3$ and $4$ bits per coordinate.

We refer the reader to~\cite{konecy2017federated,kairouz2019advances,Xu2020CompressedCF,wang2021field} for an extensive review of the current state of the art and challenges.

\section{Analysis of the Bounded Support Quantization technique}\label{app:BSQ}
In this appendix, we analyze the Bounded Support Quantization (BSQ) approach that sends all coordinates {outside a range $[-t_p,t_p]$ exactly and performs a standard (i.e., uniform) stochastic quantization for the rest.}

Let $p\in(0,1)$ and denote $t_p=\frac{\norm{\overline x}_2}{\sqrt{d\cdot p}}$; notice that there can be at most $d\cdot p$ coordinates outside $[-t_p,t_p]$.
Using $b$ bits, we split this range into $2^b-1$ intervals of size $\frac{2t_p}{2^b-1}$, meaning that each coordinate's expected squared error is at most $\parentheses{\frac{2t_p}{2^b-1}}^2/4.$
The MSE of the algorithm is therefore bounded by 
\begin{equation*}
\E\brackets{\norm{\overline x - \widehat{\overline{x}}}_2^2} = 
d\cdot \parentheses{\frac{2t_p}{2^b-1}}^2/4 =  \frac{\norm{\overline x}_2^2}{p\cdot\parentheses{ 2^b-1}^2}.
\end{equation*}

This gives the result
\begin{equation*}
\vnmse \le \frac{1}{p\cdot\parentheses{2^b-1}^2}.
\end{equation*}

Thus, as clients use independent randomness for the quantization, we have that 
\begin{equation*}
\nmse \le \frac{1}{n\cdot p\cdot\parentheses{2^b-1}^2}.
\end{equation*}

Let $r$ be the representation length of each coordinate in the input vector (e.g., $r=32$ for single-precision floats) and $i$ be the number of bits that represent a coordinate's index (e.g., $i=32$, assuming $\log{d} \le 32$). Then, we get that BSQ sends a message with less than $p\cdot (r+i) + b$ bits per coordinate.
Further, this method has $O(d)$ time for encoding and decoding and is GPU-friendly. 

As mentioned in \cref{sec:alg:bsq}, it is possible to encode the indices of the exactly sent coordinates using only $\log{\binom{d}{d \cdot p}}$ bits at the cost of additional complexity. Also, it is possible to send a bit vector to indicate whether each coordinate is exactly sent or quantized and obtain a message with fewer than $p\cdot r + b + 1$ bits. 

However, empirically we find the method of transmitting the indices without encoding most useful as $p\cdot \log d \ll 1$ in our settings, resulting in fast processing time and small bandwidth overhead. 


\section{\alg's \nmse Proof}\label{app:vnmse_proof}
In this appendix, we analyze the \vnmse and then the\nmse of our algorithm. 

Let $\chi = \E[(Z - \widehat Z)^2]$ denote the error of the quantization of a normal random variable $Z\sim\mathcal N(0,1)$. Our analysis is general and covers \alg{}, but is also applicable to any unbiased quantization method that is used following a uniform random rotation preprocessing.

Essentially, we show that \alg{}'s \vnmse is  $\chi$ plus a small additional additive error term (arising because the rotation does not yield exactly normally distributed and independent coordinates) that quickly tends to $0$ as the dimension increases. 

\begin{lemma}\label{lem:vnmse}
 For \alg{}, it holds that: 
 \begin{equation*}
 \vnmse \le \chi + O\parentheses{\sqrt{\frac{\log d}{d}}}~.
 \end{equation*}
\end{lemma}

\begin{proof}
The proof follows similar lines to that of \cite{vargaftik2021drive,EDEN}. However, here the \vnmse expression is different and is somewhat simpler as it takes advantage of our \emph{unbiased} quantization technique.

A rotation preserves a vector's euclidean norm. Thus, according to Algorithms \ref{alg:quickfl_initial_new} and \ref{alg:final} it holds that 
\begin{equation}
\begin{aligned}
        \norm{\overline x-\widehat{\overline x}}_2^2 =& \norm{T\parentheses{\overline x-\widehat{\overline x}}}_2^2 =  \norm{T\parentheses{\overline x}-T\parentheses{\widehat{\overline x}}}_2^2 = \\
        &\norm{\frac{\norm {\overline x}_2}{\sqrt d} \cdot \overline{Z} - \frac{\norm {\overline x}_2}{\sqrt d} \cdot \widehat{\overline{Z}}}_2^2 = \frac{\norm {\overline x}_2^2}{d} \cdot \norm{\overline{Z} - \widehat{\overline{Z}}}_2^2~. 
\end{aligned}
\end{equation}
Taking expectation and dividing by $\norm {\overline x}_2^2$ yields
\begin{equation}
\begin{aligned}
        \vnmse \triangleq \E\brackets{\frac{\norm{\overline x-\widehat{\overline x}}_2^2}{\norm{\overline x}_2^2}} =& \frac{1}{d} \cdot \E\brackets{\norm{\overline{Z} - \widehat{\overline{Z}}}_2^2}  \\ =&\frac{1}{d} \cdot \E\brackets{\sum_{i=0}^{d-1} \parentheses{\overline{Z}[i] - \widehat{\overline{Z}}[i]}^2} = \frac{1}{d} \cdot \sum_{i=0}^{d-1} \E\brackets{\parentheses{\overline{Z}[i] - \widehat{\overline{Z}}[i]}^2}~. 
\end{aligned}
\end{equation}

Let $\overline{\widetilde{Z}}$ be a vector of $d$ independent $\mathcal{N}(0,1)$ random variables. Then the distribution of each transformed and scaled coordinate $\overline{Z}[i]$ is given by $\overline{{Z}}[i] \sim \sqrt d \cdot\frac{\overline{\widetilde{Z}}[i]}{\norm{\overline{\widetilde{Z}}}_2}$ (e.g., see \cite{vargaftik2021drive,muller1959note}).

This means that all coordinates of $\overline{Z}$ follow the same distribution, and thus all coordinates of $\widehat{\overline{Z}}$ follow the same (different) distribution. Thus, without loss of generality, we obtain 

\begin{equation}
\begin{aligned}
        \vnmse \triangleq \E\brackets{\frac{\norm{\overline x-\widehat{\overline x}}_2^2}{\norm{\overline x}_2^2}} &= \E\brackets{\parentheses{\overline{Z}[0] - \widehat{\overline{Z}}[0]}^2} = \E\brackets{\parentheses{\frac{\sqrt d}{\norm{\overline{\widetilde{Z}}}_2} \cdot \overline{\widetilde{Z}}[0] - \widehat{\overline{Z}}[0]}^2}~. 
\end{aligned}
\end{equation}

For some $0 < \alpha < \frac{1}{2}$, denote the event
\begin{align*}
  {\mathscr {E}}=\set{d \cdot (1-\alpha)\leq \norm{\overline{\widetilde{Z}}}_2^2\leq d \cdot (1+\alpha)}~.
\end{align*}

Let ${\mathscr {E}}^c$ be the complementary event of ${\mathscr {E}}$. By Lemma D.2 in \cite{EDEN} it holds that $\Pr[{\mathscr {E}}^c] \le 2\cdot e^{-\frac{\alpha^2}{8}\cdot d}$~. Also, by the law of total expectation

\begin{equation}
\begin{aligned}
        &\E\brackets{\parentheses{\frac{\sqrt d}{\norm{\overline{\widetilde{Z}}}_2} \cdot \overline{\widetilde{Z}}[0] - \widehat{\overline{Z}}[0]}^2} \le \\
        &\E\brackets{\parentheses{\frac{\sqrt d}{\norm{\overline{\widetilde{Z}}}_2} \cdot \overline{\widetilde{Z}}[0] - \widehat{\overline{Z}}[0]}^2 \Bigg| {\mathscr {E}}}\cdot \Pr[{\mathscr {E}}] + \E\brackets{\parentheses{\frac{\sqrt d}{\norm{\overline{\widetilde{Z}}}_2} \cdot \overline{\widetilde{Z}}[0] - \widehat{\overline{Z}}[0]}^2 \Bigg| {\mathscr {E}}^c}\cdot \Pr[{\mathscr {E}}^c] \le \\ &\E\brackets{\parentheses{\frac{\sqrt d}{\norm{\overline{\widetilde{Z}}}_2} \cdot \overline{\widetilde{Z}}[0] - \widehat{\overline{Z}}[0]}^2 \Bigg| {\mathscr {E}}} \cdot \Pr[{\mathscr {E}}] + M \cdot \Pr[{\mathscr {E}}^c]
        ~, 
\end{aligned}
\end{equation}
where $M=\parentheses{\vnmse_{\max}}^2$ and $\vnmse_{\max}$ is the maximal value that the \receiver can reconstruct (i.e., $\max(Q_{b,p})$ in \cref{alg:quickfl_initial_new} or $\max(R)$ in \cref{alg:final}) which is a constant that is \emph{independent} of the vector's dimension. 
Next,

\begin{equation}
\begin{aligned}
&\E\brackets{\parentheses{\frac{\sqrt d}{\norm{\overline{\widetilde{Z}}}_2} \cdot \overline{\widetilde{Z}}[0] - \widehat{\overline{Z}}[0]}^2 \Bigg| {\mathscr {E}}} = \E\brackets{\parentheses{\parentheses{\overline{\widetilde{Z}}[0] - \widehat{\overline{Z}}[0]} + \parentheses{\frac{\sqrt d}{\norm{\overline{\widetilde{Z}}}_2}-1}\cdot \overline{\widetilde{Z}}[0]}^2 \Bigg| {\mathscr {E}}} = \\
&\E\brackets{\parentheses{\overline{\widetilde{Z}}[0] - \widehat{\overline{Z}}[0]}^2 \Bigg| {\mathscr {E}}} 
+ 2\cdot \E\brackets{\parentheses{\overline{\widetilde{Z}}[0] - \widehat{\overline{Z}}[0]} \cdot \parentheses{\frac{\sqrt d}{\norm{\overline{\widetilde{Z}}}_2}-1}\cdot \overline{\widetilde{Z}}[0] \Bigg| {\mathscr {E}}}
+\\
&\E\brackets{\parentheses{\parentheses{\frac{\sqrt d}{\norm{\overline{\widetilde{Z}}}_2}-1}\cdot \overline{\widetilde{Z}}[0]}^2 \Bigg| {\mathscr {E}}}
\end{aligned}
\end{equation}
Also,
\begin{equation}
\begin{aligned}
&\E\brackets{\parentheses{\overline{\widetilde{Z}}[0] - \widehat{\overline{Z}}[0]} \cdot \parentheses{\frac{\sqrt d}{\norm{\overline{\widetilde{Z}}}_2}-1}\cdot \overline{\widetilde{Z}}[0] \Bigg| {\mathscr {E}}} \cdot \Pr[{\mathscr {E}}] \le \\
&\parentheses{\frac{1}{\sqrt{1-\alpha}} - 1} \cdot \abs{\E\brackets{\parentheses{\overline{\widetilde{Z}}[0] - \widehat{\overline{Z}}[0]} \cdot\overline{\widetilde{Z}}[0] \big| {\mathscr {E}}} \cdot \Pr[{\mathscr {E}}]} \le \\
& \parentheses{\frac{1}{\sqrt{1-\alpha}} - 1} \cdot \abs{\E\brackets{\parentheses{\overline{\widetilde{Z}}[0]}^2 - \widehat{\overline{Z}}[0] \cdot\overline{\widetilde{Z}}[0] \bigg| {\mathscr {E}}} \cdot \Pr[{\mathscr {E}}]} \le \\
& \parentheses{\frac{1}{\sqrt{1-\alpha}} - 1} \cdot 1 + \parentheses{\frac{1}{\sqrt{1-\alpha}} - 1} \cdot  \frac{1}{\sqrt{1-\alpha}} = \frac{\alpha}{1-\alpha} \le 2\alpha~.
\end{aligned}
\end{equation}
Here, we used that $$\E\brackets{\parentheses{\overline{\widetilde{Z}}[0]}^2 \big| {\mathscr {E}}} \cdot \Pr[{\mathscr {E}}] \le \E\brackets{\parentheses{\overline{\widetilde{Z}}[0]}^2} = 1~,$$ and that 
\begin{multline}
\E\brackets{\widehat{\overline{Z}}[0] \cdot\overline{\widetilde{Z}}[0] \big| {\mathscr {E}}} \cdot \Pr[{\mathscr {E}}] = \E\brackets{\E\brackets{\widehat{\overline{Z}}[0] \cdot \overline{\widetilde{Z}}[0]  \big|  {\mathscr {E}}, \overline{\widetilde{Z}}}} \cdot \Pr[{\mathscr {E}}] \\= \E\brackets{\frac{\sqrt d}{\norm{\overline{\widetilde{Z}}}_2}\cdot\parentheses{\overline{\widetilde{Z}}[0]}^2 \big| {\mathscr {E}}}\cdot \Pr[{\mathscr {E}}] \le \frac{1}{\sqrt{1-\alpha}} \cdot \E\brackets{\parentheses{\overline{\widetilde{Z}}[0]}^2} = \frac{1}{\sqrt{1-\alpha}}.
\end{multline}
Next, we similarly obtain
\begin{equation}
\begin{aligned}
&\E\brackets{\parentheses{\parentheses{\frac{\sqrt d}{\norm{\overline{\widetilde{Z}}}_2}-1}\cdot \overline{\widetilde{Z}}[0]}^2 \Bigg|\ {\mathscr {E}}} \cdot \Pr[{\mathscr {E}}]\le 
&\parentheses{\frac{1}{\sqrt{1-\alpha}} - 1} + \parentheses{1-\frac{1}{\sqrt{1+\alpha}}} \le 2\alpha.
\end{aligned}
\end{equation}

Thus, 
\begin{equation}
\begin{aligned}
\vnmse \le\E\brackets{\parentheses{\overline{\widetilde{Z}}[0] - \widehat{\overline{Z}}[0]}^2} + 4\alpha +2\cdot e^{-\frac{\alpha^2}{8}\cdot d} \cdot M~.
\end{aligned}
\end{equation}
Setting $\alpha=\sqrt{\frac{8\log d}{d}}$ yields $\vnmse \le\E\brackets{\parentheses{\overline{\widetilde{Z}}[0] - \widehat{\overline{Z}}[0]}^2} + O\parentheses{\sqrt{\frac{\log d}{d}}}$.

Since $\overline{\widetilde{Z}}[0]\sim\mathcal N(0,1)$, we can write
\begin{align*}
    \vnmse \le\E\brackets{\parentheses{Z - \widehat Z}^2} + O\parentheses{\sqrt{\frac{\log d}{d}}}~.\qquad
\end{align*}
This concludes the proof of the Lemma.\qedhere
\end{proof}

We are now ready to prove the theorem.
\qflurrnmse*
\begin{proof}

We start by analyzing \alg's $\chi$. We can write:
\begin{multline}
    \chi = \E\brackets{\parentheses{Z {-} \widehat{Z}}^2} = 
    \E\brackets{\parentheses{Z {-} \widehat{Z}}^2 \mid Z \in [-t_p, t_p]} \cdot \Pr[ Z \in [-t_p, t_p]] \quad+ \\\E\brackets{\parentheses{Z {-} \widehat{Z}}^2 \mid Z \not\in [-t_p, t_p]}\cdot \Pr[ Z \not\in [-t_p, t_p]],    \qquad\qquad\qquad\qquad
\end{multline}
where the first summand is exactly the quantization error of our distribution-aware unbiased BSQ, and the second summand is $0$ as such values are sent exactly.

This means that for any $b$ and $p$, we can exactly compute $\chi$ given the solver's output (i.e., the precomputed quantization-values or tables). For example, it is $\approx 8.58$ for $b=1,\ell=0$ and $p=2^{-9}$. 

By \cref{lem:vnmse}, we get that \alg's \vnmse is $\chi + O\parentheses{\sqrt{\frac{\log d}{d}}} = O(1)$.

Since the clients' quantization is independent, we immediately obtain the result as $\nmse = \frac{1}{n} \cdot \vnmse$.
\end{proof}



\section{\alg{} with client-specific shared randomness}\label{app:hell}
In the most general problem formulation, we assume that the sender and receiver have access to a shared $h\sim U[0,1]$ random variable. This corresponds to having infinite shared random bits. 
Using this shared randomness, for each {message} $x \in \mathcal X_b$, the sending client chooses the probability $S(h,z,x)$ to quantize its value $z \in [-t_p, t_p]$ to the associated value $R(h,x)$ reconstructed by the receiver. We emphasize that $h$ does not need to be transmitted.
We further note that the unbiasedness constraint is now defined with respect to both the private randomness of the client (which is used to pick a message with respect to the distribution $S$) and the (client-specific) shared randomness $h$.
This yields the following optimization problem:

\resizebox{\columnwidth}{!}{
$
\begin{array}{ll@{}l}
\displaystyle{\minimize_{S,R}} & \displaystyle \textcolor{red}{\int_{0}^1}\int_{-t_p}^{t_p} \sum_{x\in \mathcal X_b} S(\textcolor{red}{h,}\  z,x) \cdot \parentheses{z-R(\textcolor{red}{h,}\ x)}^2 \cdot e^{\frac{-z^2}{2}}dz\textcolor{red}{dh\ } \\\\\bigskip
\text{subject to}\\\bigskip
{\normalfont (\textit{\textcolor{gray}{Unbiasedness}})}& \displaystyle \textcolor{red}{\int_{0}^1}\sum_{x \in \mathcal X_b} S(\textcolor{red}{h,}\ z,x) \cdot R(\textcolor{red}{h,}\ x)\textcolor{red}{\ dh} = z, & \forall z \in [-t_p, t_p]\\\bigskip
{\normalfont (\textit{\textcolor{gray}{Probability}})}&\displaystyle \sum_{x\in \mathcal X_b}S(\textcolor{red}{h,}\ z,x)=1,&
                                                                \forall\textcolor{red}{h \in [0,1],}\,\, z \in [-t_p, t_p]
                                                                \\
&S(\textcolor{red}{h,}\ z,x)\ge0,&
                                                                \forall \textcolor{red}{h \in [0,1],}\,\, z \in [-t_p, t_p], \,\, x \in \mathcal X_b

\end{array}
$
}

As in the case without shared randomness, we are unaware of analytical methods for solving this continuous problem. Therefore, we discretize it to get a problem with finitely many variables. To that end, we further discretize the client-specific shared randomness, allowing $h\in\mathcal H_\ell=\set{0,\ldots,2^\ell-1}$ to have $\ell$ shared random bits. As with the number of quantiles $m$, the parameter $\ell$ gives a tradeoff \mbox{between the complexity of the resulting (discretized) problem and the error of the quantization.}

\newpage

{We give the formulation below (with the differences from the no-client-specific-shared-randomness version highlighted in red.)}

{\small
\begin{equation*}
\begin{array}{ll@{}l}
\displaystyle{\minimize_{S',R}} & \displaystyle \sum_{\substack{\textcolor{red}{h\in\mathcal H_\ell}\\i\in\mathcal I_m\\x\in \mathcal X_b}} S'(\textcolor{red}{h,}\  i,x) \cdot \parentheses{\mathcal A_{p,m}(i)-R(\textcolor{red}{h,}\ x)}^2 \\\\\bigskip
\text{subject to}\\\bigskip
{\normalfont (\textit{\textcolor{gray}{Unbiasedness}})}\quad \mbox{\LARGE\textcolor{red}{$\frac{1}{2^{\ell}}\ \ \cdot$} }&\displaystyle\sum_{\substack{\textcolor{red}{h\in\mathcal H_\ell}\\x \in \mathcal X_b}} S'(\textcolor{red}{h,}\ i,x) \cdot R(\textcolor{red}{h,}\ x) = \mathcal A_{p,m}(i), &\qquad \forall\, i \in \mathcal I_m\\\bigskip
{\normalfont (\textit{\textcolor{gray}{Probability}})}&\displaystyle \sum_{x\in \mathcal X_b}S'(\textcolor{red}{h,}\ i,x)=1,&\qquad
                                                                \forall\,\textcolor{red}{h \in \mathcal H_\ell,}\,\, i \in \mathcal I_m
                                                                \\
&S'(\textcolor{red}{h,}\ i,x)\ge0,&\qquad
                                                                \forall\, \textcolor{red}{h \in \mathcal H_\ell,}\,\,i \in \mathcal I_m, \,\, x \in \mathcal X_b

\end{array}
\end{equation*}
}

Unlike without client-specific shared randomness, the solver's output does not directly yield an implementable algorithm, as it only associates probabilities to each $\angles{h,i,x}$ tuple. A natural option is to first stochastically quantize every rotated coordinate $Z\in[-t_p,t_p]$ to a one of the two closest quantiles before running the algorithm that is derived from solving the discrete optimization problem. The resulting pseudocode is shown in \cref{alg:quickfl_sr_sq}.

\begin{algorithm}[t]
\caption{\mbox{QUIC-FL with client-specific shared randomness and stoch. quantizing to quantiles}}
\begin{algorithmic}
    \vspace{0.5mm} \State \hspace*{-4mm}\textbf{Input:} Bit budget $b$, shared random bits $\ell$, BSQ parameter $p$ and its threshold $t_p$ and precomputed quantiles $\mathcal A_{p,m}$, sender table $S$ and receiver table $R$.
    \vspace{-2mm}\\\hspace*{-4mm}\hrulefill
    \State \hspace*{-4mm}\textbf{\Sender{} $c$:}\smallskip
    \State \,\,1.\,\,$\overline Z_c \leftarrow \frac{\sqrt d}{\norm {\overline x_c}_2}\cdot T\parentheses{\overline x_c}$\textcolor{white}{$\big($}\smallskip
    \State \,\,2.\,\,$\overline U_c, \overline I_c\leftarrow\set{ \overline{Z}_c[i] \,\big|\, \abs{\overline{Z}_c[i]} > t_p}, \set{ i \,\big|\, \abs{\overline{Z}_c[i]} > t_p}$\smallskip
    \State \,\,3.\,\,$\overline V_c\leftarrow\set{z \in \overline Z_c \big|\, |z| \le t_p}$\smallskip
    \State \,\,4. $\widetilde{\overline V}_c \leftarrow$ Stochastically quantize $\overline V_c$ using $\mathcal A_{p,m}$\medskip
    \State \,\,5. $\overline H_c \leftarrow \set{\forall i: \mbox{Sample }\overline H_c[i]\sim \mathcal U[\mathcal H_\ell]}$ \smallskip
    \State \,\,6. $\overline X_c \leftarrow \set{\forall i: \mbox{Sample }\overline X_{c}[i]\sim \set{x\ \text{with prob.}\ S({\overline{H}_c[i],\widetilde{\overline V}_c[i],x}) \mid x\in\mathcal X_b}}$ \smallskip
    \State \,\,7. Send $\parentheses{\norm {\overline x_c}_2,\,\overline X_c,\,\overline U_c,\,\overline I_c}$ to \receiver
    \\\hspace*{-4mm}\hrulefill
    %
    \State \hspace*{-4mm}\textbf{\Receiver:}\medskip
    \State \,\,8.\,\, For all $c$:\smallskip 
    \State \,\,9.\,\, \hspace{5.2mm}$\overline H_c \leftarrow \set{\forall i: \mbox{Sample }\overline H_c[i]\sim\mathcal U[\mathcal H_\ell]}$ \smallskip
    \State \,10.\,\, \hspace{3.4mm} $\widehat{\overline V}_{c} \leftarrow  \set{\forall i: R(\overline H_c[i],\overline X_c[i])}$\smallskip
    \State \,11.\,\, \hspace{3.4mm} $\widehat{\overline Z}_{c} \leftarrow$ Merge $\widehat{\overline V}_{c}$ and $\parentheses{\overline U_c,\,\overline I_c}$\smallskip
    \State \,12.\,\, $\widehat{\overline Z}_{\mathit{avg}} \leftarrow \frac{1}{n}\cdot \sum_{c=0}^{n-1} \frac{\norm {\overline x_c}_2}{\sqrt d} \cdot \widehat{\overline Z}_{c}$\smallskip
    \State \,13.\,\, $\widehat {\overline x}_{\mathit{avg}}  \leftarrow T^{-1}\parentheses{\widehat{\overline Z}_{\mathit{avg}}}$

\end{algorithmic}
\label{alg:quickfl_sr_sq}
  

\end{algorithm}
%
%
%

The resulting algorithm is near-optimal in the sense that as the number of quantiles and shared random bits tend to infinity, we converge to an optimal algorithm. In practice, the solver is only able to produce an output for finite $m,\ell$ values; this means that the algorithm would be optimal if coordinates are uniformly distributed over $\mathcal A_{p,m}$.

In words Algorithm~\ref{alg:quickfl_sr_sq} starts similarly to~\cref{alg:quickfl_initial_new} by transforming and scaling the vector before splitting it to the large coordinates (that are sent accurately along with their indices) and the small coordinates (that are to be quantized). The difference is in the quantization process;
\cref{alg:quickfl_sr_sq} first stochastically quantizes each small coordinate to a quantile in $\mathcal A_{p,m}$. Next, the client generates the (client-specific) shared randomness $\overline H_c$ and uses the pre-computed table $S$ to sample a message for each coordinate. That is, for each coordinate $i$, knowing the shared random value $\overline H_c[i]$ and the (rounded-to-quantile) transformed coordinate $\widetilde{\overline V}_c[i]$, for all $x\in\mathcal X_b$, $S({\overline{H}_c[i],\widetilde{\overline V}_c[i],x})$ is the probability that the client should send the message $x$. We note that the message for the $i$'th coordinate is sampled from ${x~\mbox{w.p.}~ S({\overline{H}_c[i],\widetilde{\overline V}_c[i],x})}$ using the client's private randomness. Finally, the client sends its vector's norm, the sampled messages, and the values and indices of the large transformed coordinates.

In turn, the server's algorithm is also similar to~\cref{alg:quickfl_initial_new}, except for the estimation of the small transformed coordinates. In particular, for each client $c$, the server generates the client-specific shared randomness $\overline H_c$ and uses it to estimate each transformed coordinate $i$ using $R(\overline H_c[i],\overline X_c[i])$.

\subsection{Interpolating the Solver's Solution}\label{app:alg2}
A different approach, based on our examination of solver outputs, to yield an implementable algorithm from the optimal solution to the discrete problem is to calculate the message distribution directly from the rotated values without stochastically quantizing as we do in~\cref{alg:quickfl_sr_sq}.
Indeed,  we have found this approach somewhat faster and more accurate.    

A crucial ingredient in getting a human-readable solution from the solver is that we, without loss of generality, force monotonicity in both $h$ and $x$, i.e., $(x\ge x')\wedge(h\ge h')\implies R(h,x)\ge R(h',x').$
We further found symmetry in the optimal sender and receiver tables for small values of $\ell$ and $m$.
We then forced this symmetry to reduce the complexity of the solver's optimization problem size for larger $\ell$ and $m$ values. 
We use this symmetry in our interpolation.  
\paragraph{Examples, intuition and pseudocode.}\label{sec:interpolation}
We first explain the process by considering an example. We consider the setting of $p=\frac{1}{512}$ ($t_p\approx3.097$), $m=512$ quantiles, $b=2$ bits per coordinate, and $\ell=2$ bits of shared randomness.
The solver's solution for the \receiver's table $R$ is given below:

{\small
\begin{table}[H]
\centering
\begin{tabular}{|l|l|l|l|l|}
\hline
& $x=0$ & $x=1$  & $x=2$ & $x=3$ \\ \hline
$h=0$     & -5.48 & -1.23  & \textbf{0.164} & 1.68  \\ \hline
$h=1$     & -3.04 & -0.831 & \textbf{0.490} & 2.18  \\ \hline
$h=2$     & -2.18 & \textbf{-0.490} & 0.831 & 3.04  \\ \hline
$h=3$     & -1.68 & \textbf{-0.164} & 1.23  & 5.48  \\ \hline
\end{tabular}\vspace*{0.5mm}
\caption{Optimal \receiver values ($R({h,x})$) for $x\in\mathcal X_2, h\in\mathcal H_2$ when $p=1/512$ and $m=512$, rounded to $3$ significant digits.} \label{tbl:receiver}
\end{table}
}

The way to interpret the table is that if the server receives a message $x$ and the shared random value was $h$, it should estimate the (quantized) coordinate value as $R(h,x)$. For example, if $x=h=2$, the estimated value would be $0.831$. We now explain what the table means for the sending client, starting with an example.

Consider $\overline{V}_c[i] = 0$. The question is: what message distribution should the sender use, given that $\overline{V}_c[i] \notin \mathcal{A}_{p,m}$  (and without quantizing the value to a quantile)? 
Based on the shared randomness value, we can use 
$$\overline X_c[i]=\begin{cases}
1 & \mbox{If $\overline{H}_c[i] > 1$}\\
2 & \mbox{Otherwise}
\end{cases}.$$ 
%
%
Indeed, we have that the estimate is unbiased as the receiver will estimate one of the bold entries in~\cref{tbl:receiver} with equal probabilities, i.e., $\mathbb E\brackets{\widehat {\overline{V}}_c[i]} = \frac{1}{4}\sum_{\overline{H}_c[i]}R(\overline{H}_c[i],\overline X_c[i])=0$. 


Now, suppose that $\overline{V}_c[i]\in(0,t_p]$ (the case $\overline{V}_c[i]\in[-t_p,0)$ is symmetric). The \sender can increase the \receiver estimate's expected value (compared with the above choice of $\overline X_c[i]$'s distribution for $\overline{V}_c[i]=0$) by moving probability mass to larger $\overline{X}_c[i]$ values for some (or all) of the options for $\overline X_c[i]$.

For any $\overline{V}_c[i]\in(-t_p, t_p)$, there are infinitely many \sender alternatives that would yield an unbiased estimate. 
For example, if $\overline{V}_c[i]=0.1$, below are two \sender options (rounded to three significant digit):
%
\begin{align*}
&S_1(\overline{H}_c[i],\overline{V}_c[i],\overline{X}_c[i])\approx\begin{cases}
1 &   \mbox{If $(\overline{X}_c[i]=1\wedge \overline{H}_c[i]\le 2)$}\\
0.595 & \mbox{If $(\overline{X}_c[i]=2\wedge \overline{H}_c[i]=3)$}\\
0.405 & \mbox{If $(\overline{X}_c[i]=3\wedge \overline{H}_c[i]=3)$}\\
0 & \mbox{Otherwise}
\end{cases}
\\\\
&S_2(\overline{H}_c[i],\overline{V}_c[i],\overline{X}_c[i])\approx\begin{cases}
1 &   \mbox{If $(\overline{X}_c[i]=2\wedge \overline{H}_c[i]\le 1) \vee (\overline{X}_c[i]=1\wedge \overline{H}_c[i]=3)$}\\
0.697 & \mbox{If $(\overline{X}_c[i]=1\wedge \overline{H}_c[i]=2)$}\\
0.303 & \mbox{If $(\overline{X}_c[i]=2\wedge \overline{H}_c[i]=2)$}\\
0 & \mbox{Otherwise}
\end{cases}
\end{align*}
%
Note that while both $S_1$ and $S_2$ produce unbiased estimates, their expected squared errors differ.
Further, since $0.1\not\in\mathcal A_{p,m}$, the solver's output does not directly indicate what is the optimal message distribution, even though the \receiver table is known.

The approach we take corresponds to the following process. We move probability mass from the \emph{leftmost, then uppermost} entry with non-zero mass to its right neighbor in the \receiver table. So, for example, in Table~\ref{tbl:receiver}, as $\overline{V}_c[i]$ increases from 0, we first move mass from the entry $\overline{H}_c[i]=2, \overline{X}_c[i]=1$ to the entry $\overline{H}_c[i]=2, \overline{X}_c[i] =2$.  That is, the \sender, based on its private randomness, increases the probability of message $\overline{X}_c[i]=2$ and decreases the probability of message $\overline{X}_c[i]=1$ when $\overline{H}_c[i]=2$. The amount of mass moved is always chosen to maintain unbiasedness.
At some point, as $\overline{V}_c[i]$ increases, all of the probability mass will have moved, and then we start moving mass from $\overline{H}_c[i]=3, \overline{X}_c[i]=1$ similarly. (And subsequently, from $\overline{H}_c[i]=0,\overline{X}_c[i]=2$ and so on.) 

This process is visualized in Figure~\ref{fig:solvers_alg}. Note that $S(\overline{H}_c[i],\overline{V}_c[i],\overline{X}_c[i])$ values are piecewise linear as a function of $\overline{V}_c[i]$, and further, these values either go from 0 to 1, 1 to 0, or 0 to 1 and back again (all of which follow from our description). We can turn this description into formulae as explained below.

\begin{figure*}[h]
\centering
\includegraphics[clip, trim=0cm 2.2cm 0cm 0cm, width=0.9985\linewidth]{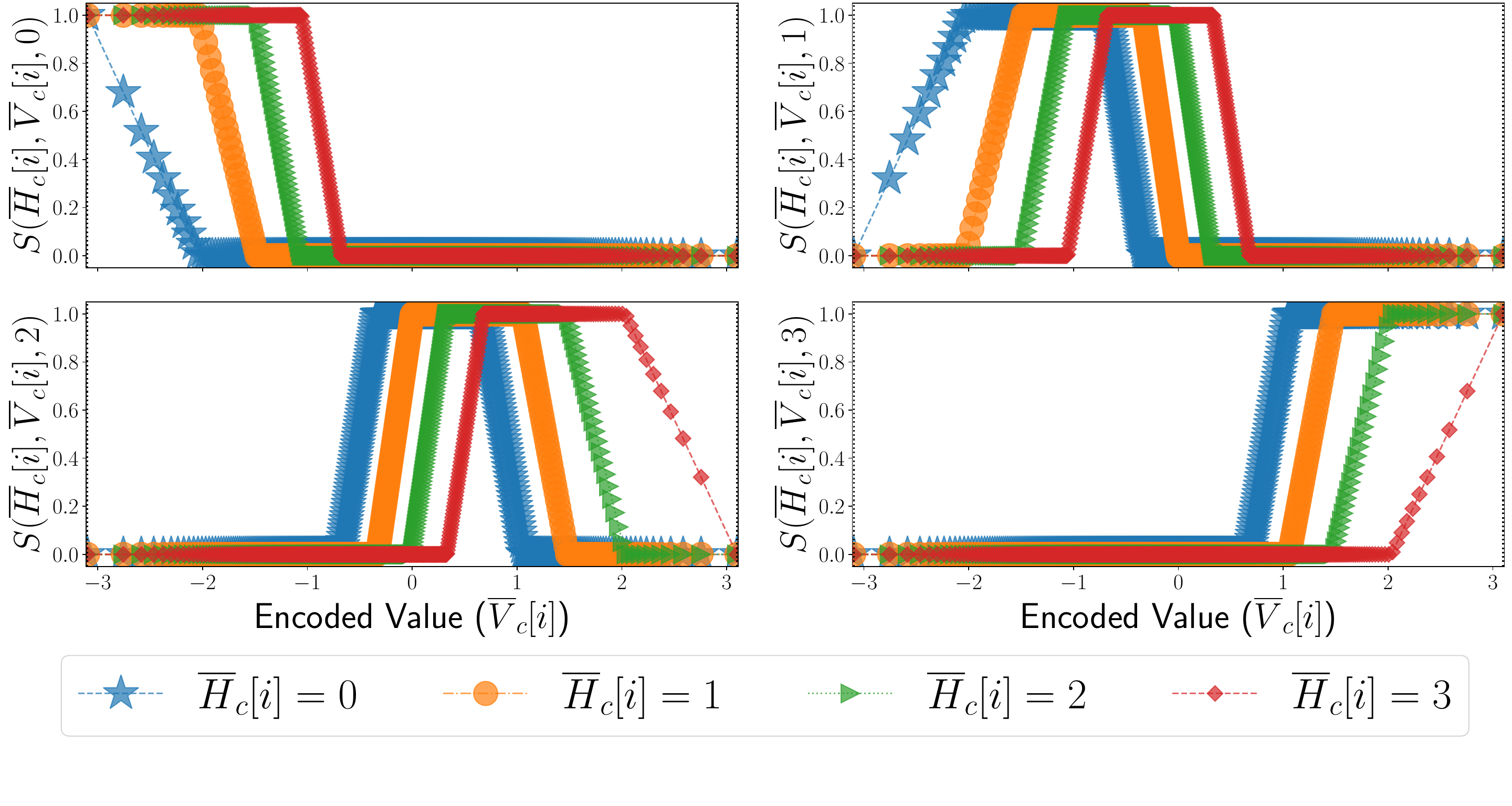}
\caption{The interpolated solver's \sender algorithm for $b=\ell=2, m=512, p=\frac{1}{512}$. 
{Markers correspond to quantiles in $\mathcal A_{p,m}$, and the lines illustrate our interpolation.}\label{fig:solvers_alg}}
\end{figure*}

\looseness=-1
\paragraph{Derivation of the interpolation equations.}
We have found, by applying the mentioned monotonicity constraints (i.e., $(x\ge x')\wedge(h\ge h')\implies R(h,x)\ge R(h',x')$) and examining the solver's solutions for our parameter range, that the optimal approach for the \sender has a structure that we can generalize beyond specific examples.
Namely, when the \receiver table is monotone, the optimal solution \emph{deterministically} quantizes the message to send in all but (at most) one shared randomness value. For instance,  $S_2$ in the example above deterministically quantizes the message if $\overline{H}_c[i]\neq 2$ (sending $\overline{X}_c[i]=1$ if $\overline{H}_c[i]=3$ or $\overline{X}_c[i]=2$ if $\overline{H}_c[i]\in\set{0,1}$), or stochastically quantizes between $\overline{X}_c[i]=1$ and $\overline{X}_c[i]=2$ when $\overline{H}_c[i]=2$. 
{Furthermore, the shared randomness value in which we should stochastically quantize the message is easy to calculate.}

To capture this behavior, we define the following quantities:
\begin{itemize}[align=left, leftmargin=0mm, labelindent=0\parindent, listparindent=0\parindent, labelwidth=0mm,itemindent=!,itemsep=1pt,parsep=0pt,topsep=1pt]
    \item The minimal message $\overline{X}_c[i]$ the \sender may send for $\overline{V}_c[i]$: 
    \begin{equation*}
    \underline x(\overline{V}_c[i])=\max\set{\ x\in\mathcal X_b \quad\bigg | \quad\parentheses{\frac{1}{2^\ell}\cdot\sum_{\overline{H}_c[i]\in\mathcal H_\ell}R(\overline{H}_c[i],x)}\le \overline{V}_c[i]\ }.\end{equation*}
    That is, $\underline x(\overline{V}_c[i])$ is the maximal value such that sending $\underline x(\overline{V}_c[i])$ regardless of the shared randomness value would result in not overestimating $\overline{V}_c[i]$ in expectation.
    For example, as illustrated in~\cref{tbl:receiver} ($b=\ell=2$), we have $\underline x(0)=1$, as the \sender sends either $1$ or $2$ (highlighted in bold) depending on the shared randomness value.
    \item For convenience, we denote $R(h,2^b)=\infty$ for all $h\in\mathcal H_\ell$.  
    Then, the shared randomness value for which the sender stochastically quantizes is given by:
    \begin{equation*}
\underline h(\overline{V}_c[i]) = \max \set{h\in\mathcal H_\ell \,\,\Bigg |\,\, \frac{1}{2^\ell}\cdot\parentheses{\sum_{h'=0}^{h-1}R(h',\underline x(\overline{V}_c[i])+1) + \sum_{h'=h}^{2^\ell-1} R(h',\underline x(\overline{V}_c[i]))}\le \overline{V}_c[i]}\,.
\end{equation*}
    That is, $\underline h(\overline{V}_c[i])$ denotes the maximal value for which sending $\Big(\underline x(\overline{V}_c[i])+1\Big)$ if $\overline{H}_c[i]<\underline h(\overline{V}_c[i])$ or $\underline x(\overline{V}_c[i])$ if $\overline{H}_c[i]\ge\underline h(\overline{V}_c[i])$ would not overestimate $\overline{V}_c[i]$ in expectation. In the same example of~\cref{tbl:receiver} ($b=\ell=2$), we have $\underline h(0)=2$ since sending $\overline{X}_c[i]=2$ for $h\le 2$ would result in an overestimation.
\end{itemize}

%

\paragraph{The sender-interpolated algorithm.}

Let us denote by $\mu$ the expectation we require for $\overline{H}_c[i]=\underline h(\overline{V}_c[i])$ to ensure that our algorithm is unbiased:
\begin{multline*}
    \overline \mu_c[i] \triangleq \mathbb  E\brackets{\ \widehat {\overline{V}}_c[i]\ \big |\ \ \overline{H}_c[i] = \underline h(\overline{V}_c[i])\ } = \\ 
    \ 2^{\ell}\cdot \overline{V}_c[i] - \sum_{h=0}^{\underline h(\overline{V}_c[i])-1}R\parentheses{{h,\underline x(\overline{V}_c[i])+1}} + \sum_{h=\underline h(\overline{V}_c[i])+1}^{2^\ell-1}R\parentheses{{h,\underline x(\overline{V}_c[i])}}\ \ .
\end{multline*} 
We further make the following definitions:
\begin{itemize}
    \item  The probability of rounding the message up to $\underline x(\overline{V}_c[i])+1$ when $\overline{H}_c[i] = \underline h$:
    \begin{equation*}
\overline p_{c}[i] = \frac{\overline \mu_c[i]-R({\overline{H}_c[i],\underline x(\overline{V}_c[i])})}{R({\overline{H}_c[i],\underline x(\overline{V}_c[i])}+1)-R({\overline{H}_c[i],\underline x(\overline{V}_c[i])})}
\end{equation*}
\item The probability of rounding the message down to $\underline x(\overline{V}_c[i])$ when $\overline{H}_c[i] = \underline h$:
\begin{equation*}
\overline q_{c}[i] = 1 - \overline p_{c}[i] = \frac{R({\overline{H}_c[i],\underline x(\overline{V}_c[i])}+1)-\overline \mu_c[i]}{R({\overline{H}_c[i],\underline x(\overline{V}_c[i])}+1)-R({\overline{H}_c[i],\underline x(\overline{V}_c[i])})}.
\end{equation*}

\end{itemize}

Then, for any shared randomness value $\overline{H}_c[i]\in\mathcal H_\ell$, to-be-quantized value $ \overline{V}_c[i]\in[-t_p,t_p]$, and message $x \in \mathcal X_b$, the interpolated algorithm works as follows:
\begin{align}
S(\overline{H}_c[i],\overline{V}_c[i],x)=\begin{cases}
1 & \mbox{If $\parentheses{x=\underline x(\overline{V}_c[i]) \wedge \overline{H}_c[i]>\underline h}\vee \parentheses{x=\underline x(\overline{V}_c[i])+1 \wedge \overline{H}_c[i]<\underline h}$}\\
\overline p_{c}[i] & \mbox{If $(x=\underline x(\overline{V}_c[i])+1 \wedge \overline{H}_c[i] = \underline h)$}\vspace{1mm}\\
\overline q_{c}[i] & \mbox{If $(x=\underline x(\overline{V}_c[i]) \wedge \overline{H}_c[i] = \underline h)$}\\
0 & \mbox{Otherwise}
\end{cases}\quad.\label{Eq:SQ_expectation}
\end{align} 

Namely, if $\overline{H}_c[i]<\underline h$, the \sender deterministically sends $\parentheses{\underline x(\overline{V}_c[i])+1}$ and if $\overline{H}_c[i]>\underline h$, {the \sender deterministically sends ${\underline x(\overline{V}_c[i])}$. Finally, if $\overline{H}_c[i] = \underline h$, it sends $\parentheses{\underline x(\overline{V}_c[i])+1}$ with probability $\overline p_{c}[i]$ and ${\underline x(\overline{V}_c[i])}$ otherwise.
Indeed, by our choice of $\overline \mu_c[i]$, the algorithm is guaranteed to be unbiased for all $\overline{V}_c[i]\in[-t_p,t_p]$.


The pseudocode of this variant is given by~\cref{alg:final}.

\begin{algorithm}[]
\small
\caption{QUIC-FL with client-specific shared randomness and \sender interpolation}
\begin{algorithmic}
    \vspace{0.5mm} \State \hspace*{-4mm}\textbf{Input:} Bit budget $b$, shared random bits $\ell$, BSQ parameter $p$ and its threshold $t_p$ and precomputed quantiles $\mathcal A_{p,m}$, and receiver table $R$. (The table $S$ is not needed.)
    \vspace{-2mm}\\\hspace*{-4mm}\hrulefill
    \State \hspace*{-4mm}\textbf{\Sender{} $c$:}\smallskip
    \State \,\,1.\,\,$\overline Z_c \leftarrow \frac{\sqrt d}{\norm {\overline x_c}_2}\cdot T\parentheses{\overline x_c}$\textcolor{white}{$\big($}\smallskip
    \State \,\,2.\,\,$\overline U_c, \overline I_c\leftarrow\set{ \overline{Z}_c[i] \,\big|\, \abs{\overline{Z}_c[i]} > t_p}, \set{ i \,\big|\, \abs{\overline{Z}_c[i]} > t_p}$\smallskip
    \State \,\,3.\,\,$\overline V_c\leftarrow\set{z \in \overline Z_c \big|\, |z| \le t_p}$\smallskip
    \State \,\,4. $\overline H_c \leftarrow \set{\forall i: \mbox{Sample }\overline H_c[i]\sim \mathcal U[\mathcal H_\ell]}$ \smallskip
    \State \,\,5. $\overline X_c \leftarrow \set{\forall i: \mbox{Sample }\overline X_{c}[i]\sim \set{x\ \text{with prob.}\ S({\overline{H}_c[i],\overline{V}_c[i],x})}}$ \quad\textcolor{gray}{\hfill$\triangleright$\,\, According to~\cref{Eq:SQ_expectation}}\smallskip
    \State \,\,6. Send $\parentheses{\norm {\overline x_c}_2,\,\overline X_c,\,\overline U_c,\,\overline I_c}$ to \receiver
    \\\hspace*{-4mm}\hrulefill
    %
    \State \hspace*{-4mm}\textbf{\Receiver:}\medskip
    \State \,\,7.\,\, For all $c$:\medskip 
    \State \,\,8.\,\, \hspace{5.2mm}$\overline H_c \leftarrow \set{\forall i: \mbox{Sample }\overline H_c[i]\sim \mathcal U[\mathcal H_\ell]}$ \smallskip
    \State \,\,9.\,\, \hspace{5.2mm}$\widehat{\overline V}_{c} \leftarrow  \set{\forall i: R(\overline H_c[i],\overline X_c[i])}$\smallskip
    \State \,10.\,\, \hspace{4.1039864mm}$ \widehat{\overline Z}_{c} \leftarrow$ Merge $\widehat{\overline V}_{c}$ and $\parentheses{\overline U_c,\,\overline I_c}$\smallskip
    \State \,11.\,\, $\widehat{\overline Z}_{\mathit{avg}} \leftarrow \frac{1}{n}\cdot \sum_{c=0}^{n-1} \frac{\norm {\overline x_c}_2}{\sqrt d} \cdot \widehat{\overline Z}_{c}$\smallskip
    \State \,12.\,\, $\widehat {\overline x}_{\mathit{avg}}  \leftarrow T^{-1}\parentheses{\widehat{\overline Z}_{\mathit{avg}}}$

\end{algorithmic}
  

\label{alg:final}
\end{algorithm}


\section{Performance of \alg with the Randomized Hadamard Transform} \label{app:hadamard}

As described earlier, while ideally we would like to use a fully random rotation on the $d$-dimensional sphere as the first step to our algorithms, this is computationally expensive.  Instead, we suggest using a randomized Hadamard transform (RHT), which is computationally more efficient. We formally show below that using RHT has the same asymptotic guarantee as with random rotations, albeit with a larger constant (constant factor increases in the fraction of exactly sent coordinates and \nmse).  
Namely, we show that (1) the expected number of transformed and scaled coordinates that fall outside $[-t_p,t_p]$ (for the same choice of $t_p$ as a function of $p$), is bounded by $3.2p$; (2) that we still get $O(1/n)$ NMSE for any $b\ge 1$. 
Further, we find that running \alg with RHT and $b+1$ bits per quantized coordinate has a lower \nmse than \alg with a uniform random rotation for $p=2^{-9}$ and any $b\in\set{1,2,3}$.

We note that some works suggest using two or three successive randomized Hadamard transforms to obtain something that should be closer to a uniform random rotation \cite{yu2016orthogonal,10.5555/2969239.2969376}.  This naturally takes more computation time. In our case, and in line with previous works~\cite{vargaftik2021drive,EDEN}, we find empirically that one RHT appears to suffice. However, unlike these works, our algorithm remains provably unbiased and maintains the $O(1/n)$ \nmse guarantee.
Determining better provable bounds using two or more RHTs is left as an open problem.  
 \begin{figure*}[h]
\centering
\hspace*{-0mm}\includegraphics[width=0.7658\linewidth]{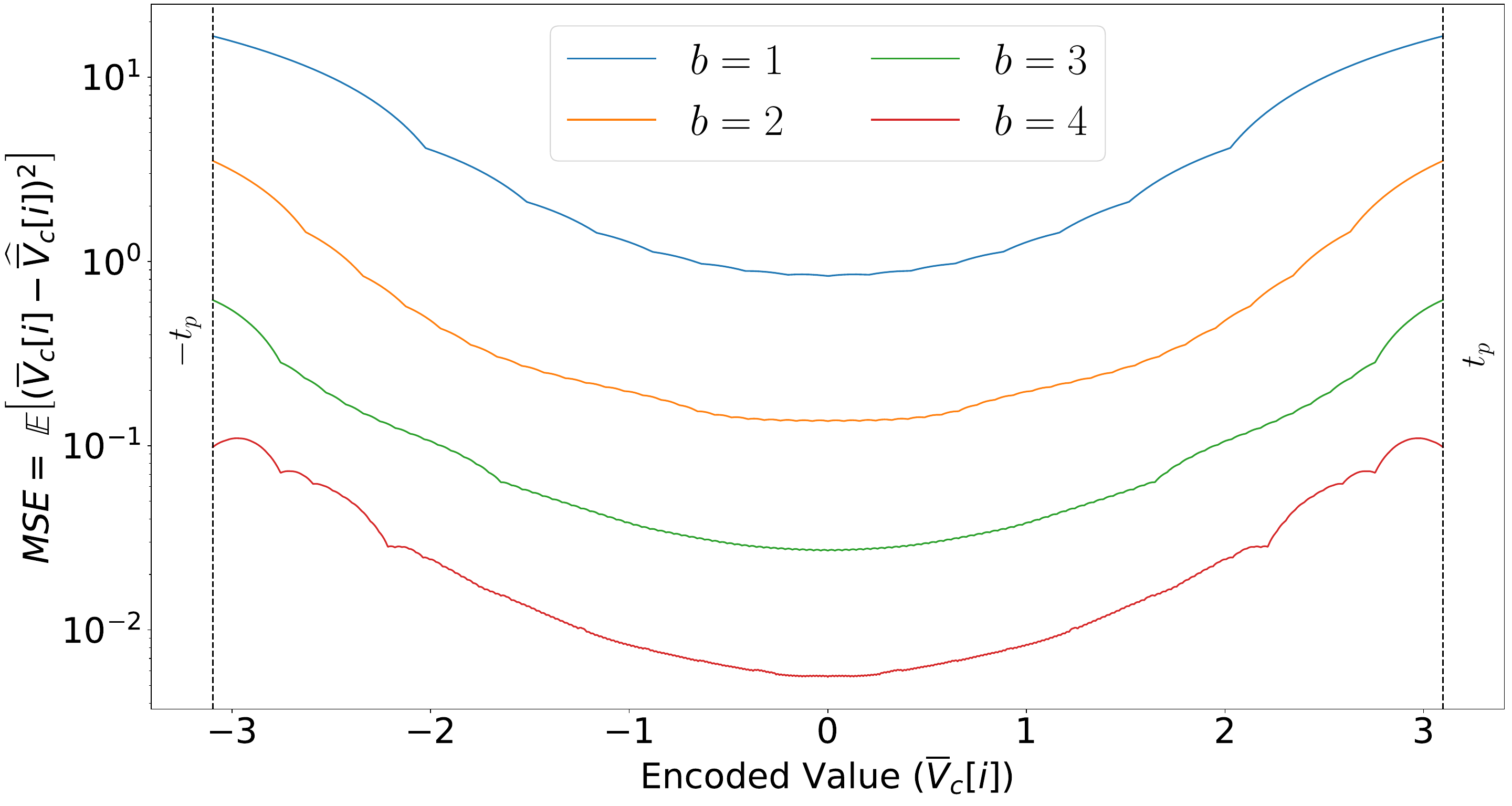}
\caption{Expected squared error as a function of the encoded value (for $p=\frac{1}{512}, m=512$).}
\label{fig:errorAsAFunctionOfZ}
\end{figure*}
\begin{theorem}
Let ${\overline x}\in\mathbb R^d$, let $T_{RHT}(\overline x)$ be the result of a randomized Hadamard transform on $\overline x$, and let $\mathfrak Z=\overline{V}_c[i]=\frac{\sqrt d}{\norm{\overline x}_2}T_{RHT}(x)[i]$ be a coordinate in the transformed and scaled vector.
For any $p$, $\Pr\brackets{\mathfrak Z\not\in[-t_p,t_p]}\le 3.2p.$
\end{theorem}
\begin{proof}
This follows from the theorem by \citet{bentkus2015tight} (\cref{thm:radamacherSum}), which we restate below.
\end{proof}
\begin{theorem}[\citet{bentkus2015tight}]\label{thm:radamacherSum}
Let $\epsilon_1,\ldots,\epsilon_d$ be i.i.d. Radamacher random variables and let $\overline a\in\mathbb R^d$ such that $\norm {\overline a}_2^2{\le} 1$. For any {\small$t\in\mathbb R$, $\Pr\brackets{\displaystyle\sum_{i=0}^{d-1} \overline a[i]\cdot \epsilon_i \ge t}\le \frac{\Pr\brackets{Z \ge t}}{4\Pr\brackets{Z \ge \sqrt 2}}\approx 3.1787\Pr\brackets{Z \ge t}$}, for $Z\sim\mathcal N(0,1)$.\qedhere
\end{theorem}

In what follows, we present a general approach to bound the quantization error of each transformed and scaled coordinate (and thus, the \alg's \nmse). 
 Our method splits $[0,t_p]$ (the argument is symmetric for $[-t_p,0]$) into several (e.g., three) intervals $\mathfrak I_0,\ldots,\mathfrak I_w$ (for some $w\in\mathbb N^+$), such that the partitioning satisfies two properties:
 \begin{itemize}
     \item The maximal error for the $i$'th interval, $\max_{z\in \mathfrak I_i}\mathbb E\brackets{\parentheses{z-\widehat{z}}^2}$, is lower than the $j$'th interval, for any $j< i$.
     \item The probability that a normal random variable $Z\sim \mathcal N(0,1)$ falls outside $\mathfrak I_0$ is less than $1/3.2$.
 \end{itemize}
 These two properties allow us to use~\cref{thm:radamacherSum} to upper bound the resulting quantization error.
 
 We exemplify the method using $p=\frac{1}{512}$, the parameter of choice for our evaluation, although it is applicable to any $p$. 
Since we believe it provides only a loose bound, we do not optimize the argument beyond showing the technique.

\begin{theorem}
Fix $p=\frac{1}{512}$; let $\overline x_c\in\mathbb R^d$ and denote by $\mathfrak Z=\overline{V}_c[i]=\frac{\sqrt d}{\norm{\overline x_c}_2}T_{RHT}(\overline x_c)[i]$ its $i$'th coordinate after applying RHT and scaling. Denoting by  $E_b = \mathbb E\brackets{(\mathfrak Z - \widehat{\mathfrak Z_b})^2}$ the mean squared error using $b$ bits per quantized coordinate, we have $E_1\le4.831$, $E_2\le0.692$, $E_3\le0.131$, $E_4\le0.0272$~.
\end{theorem}
\begin{proof}

We bound the MSE of quantizing $\mathfrak Z$, leveraging Theorem~\ref{thm:radamacherSum}.
Since the MSE, as a function of $\mathfrak Z$, is symmetric around $0$ (as illustrated in Figure~\ref{fig:errorAsAFunctionOfZ}), we analyze the $\mathfrak Z\ge 0$ case.

We split $[0,t_p]$ into intervals that satisfy the above properties, e.g., $\mathfrak I_0=[0,1.5]$, $\mathfrak I_1=(1.5,2.2]$, $\mathfrak I_2=(2.2,t_p]$. We note that this choice of intervals is not optimized and that a finer-grained partition to more intervals can improve the error bounds. 
Next, using Theorem~\ref{thm:radamacherSum}, we get that 
\begin{itemize}
    \item $P_0\triangleq\Pr[\mathfrak Z\not\in \mathfrak I_0]\le3.2\Pr[ Z\not\in \mathfrak I_0]\le 0.427$.
    \item $P_1\triangleq\Pr[\mathfrak Z\not\in (\mathfrak I_0\cup \mathfrak I_1)]\le3.2\Pr[ Z\not\in(\mathfrak I_0\cup \mathfrak I_1)]\le 0.089$.
\end{itemize}

Next, we provide the maximal error for each bit budget $b$ and such interval:
\begin{table}[h]
\centering
\begin{tabular}{|l|l|l|l|l|}
\hline
                     & ${b=1}$ & ${b=2}$ & ${b=3}$ & ${b=4}$ \\ \hline
${\mathfrak I_0}$   & $2.063$        & $0.267$        & $0.056$        & $0.0134$       \\ \hline
${\mathfrak I_1}$ & $6.39$         & $0.67$         & $0.128$        & $0.0285$       \\ \hline
${\mathfrak I_2}$   & $16.73$        & $3.51$         & $0.617$        & $0.11$         \\ \hline
\end{tabular}
\caption{For each interval $\mathfrak I_i,~i\in\set{0,1,2}$ and bit budget $b\in\set{1,2,3,4}$, depicted is the maximal MSE, i.e., $\max_{z\in \mathfrak I_i}\mathbb E\brackets{\parentheses{z-\widehat{z}}^2}$.}
\end{table}

Note that for any $b\in\set{1,2,3,4}$, the MSEs in $\mathfrak I_2$ are strictly larger than those in $\mathfrak I_1$ which are strictly larger than those in $\mathfrak I_0$.
This allows us to derive formal bounds on the error.
For example, for $b=1$, we have that the error is bounded by
\begin{equation*}
E_1\le (1-P_0)\cdot 2.063 + (P_0-P_1)\cdot 6.39 + P_1\cdot 16.73 \le 4.831.
\end{equation*}
Repeating this argument, we also obtain:
\begin{align*}
    E_2&\le (1-P_0)\cdot 0.267  + (P_0-P_1)\cdot 0.67   + P_1\cdot 3.51 \le 0.692 \\
    E_3&\le (1-P_0)\cdot 0.056  + (P_0-P_1)\cdot 0.128  + P_1\cdot 0.617 \le 0.131 \\
    E_4&\le (1-P_0)\cdot 0.0134 + (P_0-P_1)\cdot 0.0285 + P_1\cdot 0.11 \le 0.0272.\qedhere
\end{align*}
\end{proof}


\section{Shakespeare Experiments details}\label{app:expr-details}

The Shakespeare next-word prediction discussed in \S\ref{sec:expr:next-word} was first suggested in \cite{mcmahan2017communication} to naturally simulate a realistic heterogeneous federated learning setting. Its dataset consists of 18,424 lines of text from Shakespeare plays \cite{shakespeare} partitioned among the respective 715 speakers (i.e., clients). We train a standard LSTM recurrent model \cite{Hochreiter1997LongSM} with ${\approx}820K$ parameters and follow precisely the setup described in \cite{reddi2021adaptive} for the Adam server optimizer case. We restate the hyperparameters for convenience in Table~\ref{table:reddi_params}.

\begin{table*}[h]
\centering
\begin{tabular}{|l||c|c|c|c|c|c|}
\hline
\textbf{Task} & \textbf{Clients per round} & \textbf{Rounds} & \textbf{Batch size} & \textbf{Client lr} & \textbf{Server lr} & \textbf{Adam's $\epsilon$} \\ \hline \hline
Shakespeare & 10 & 1200 & 4 & 1  &  $10^{-2}$ & $10^{-3}$ \\ \hline
\end{tabular}
\caption{Hyperparameters for the Shakespeare next-word prediction experiments.}
\label{table:reddi_params} 
\end{table*}

\section{Additional Evaluation}\label{app:eval}

Our code will be released as open source upon publication.
As discussed, we use $p=1/512$, $\ell=6$ for $b=1$, $\ell=5$ for $b=2$, and $\ell=4$ for $b\in\set{3,4}$.  


\subsection{Image Classification}\label{app:imagecfl}


We evaluate \alg against other schemes with $10$ persistent clients over uniformly distributed CIFAR-10 and CIFAR-100 datasets~\cite{krizhevsky2009learning}.
We also evaluate \emph{Count-Sketch}~\cite{charikar2002finding} (denoted CS), often used for federated compression schemes (e.g.,~\cite{ivkin2019communication}) and EF21~\cite{richtarik2021ef21} a recent SOTA error-feedback framework that uses top-k as a building block with $k = 0.05 \cdot d$ (translates to 1.6 bits per-coordinate ignoring the overhead of indices encoding overhead).
For QSGD, we use twice the bandwidth of the other algorithms (one bit for sign and another for stochastic quantization). We note that QSGD also has a more accurate variant that uses variable-length encoding~\cite{NIPS2017_6c340f25}. However, it is not GPU-friendly, and therefore, as with other variable-length encoding schemes, as we have discussed previously, we do not include it in the experiment.

For CIFAR-10 and CIFAR-100, we use the ResNet-9~\cite{he2016deep} and ResNet-18~\cite{he2016deep} architectires, and use learning rates of $0.1$ and $0.05$, respectively.
For both datasets, the clients perform a single optimization step at each round. 
Our setting includes an SGD optimizer with a cross-entropy loss criterion, a batch size of 128, and a bit budget $b=1$ for the DME methods (except for EF21 and QSGD as stated above).
The results are shown in Figure~\ref{fig:cross_silo_w_qsgd}, with a rolling mean average window of 500 rounds. As shown, \alg is competitive with EDEN and the Float32 baseline and is more accurate than other methods.


%
Next, we repeat the above CIFAR-10 and CIFAR-100 experiments with the same bandwidth budgets but consider a cross-device setup with the following changes: there are $50$ clients (instead of $10$) and at each training round, $10$ out of $50$ clients are randomly selected and perform training over $5$ local steps (instead of $1$). 

Figure~\ref{fig:cross_device_w_qsgd} shows the results with a rolling mean window of 200 rounds. Again, \alg is competitive {with the asymptotically slower EDEN and the uncompressed baseline. Kashin-TF is less accurate, followed by Hadamard.}



\begin{figure*}[h!]
\centering
\includegraphics[trim={0 2.1cm 0 0},clip,width=1\linewidth]{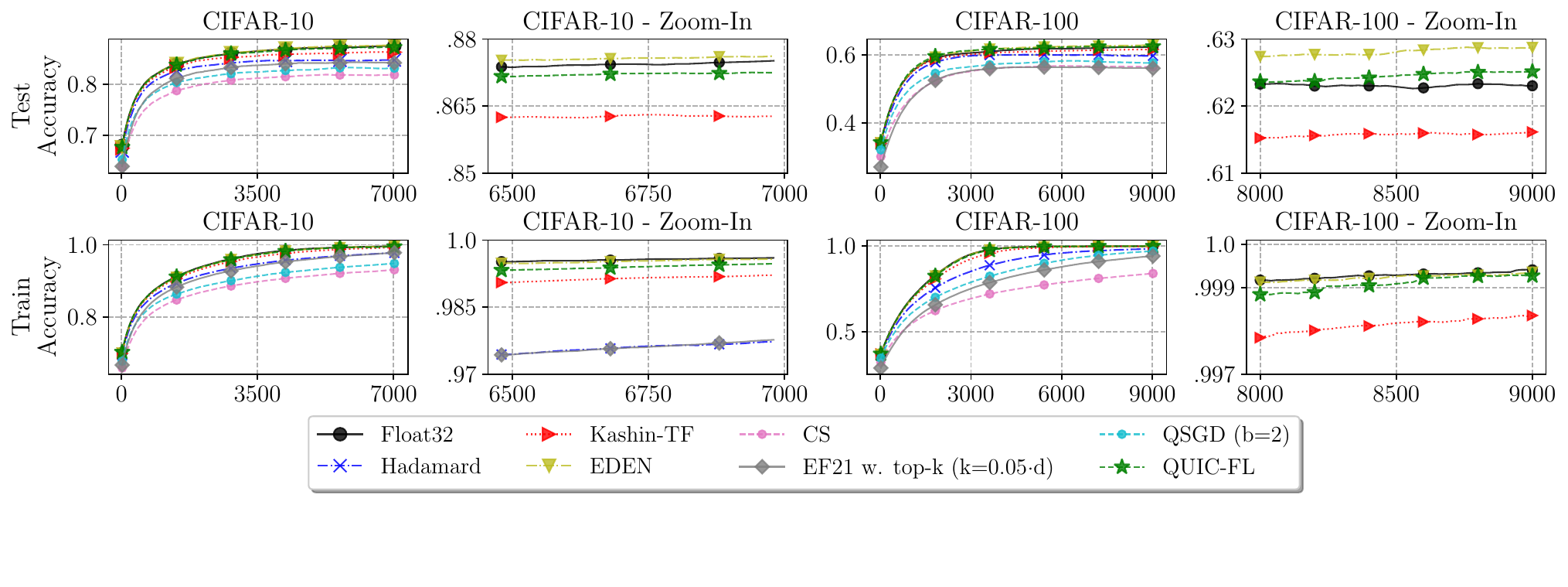}
\caption{Cross-silo federated learning.}
\label{fig:cross_silo_w_qsgd}
\end{figure*}

\begin{figure*}[h!]
\centering
\includegraphics[trim={0 2.1cm 0 0},clip,width=1\linewidth]{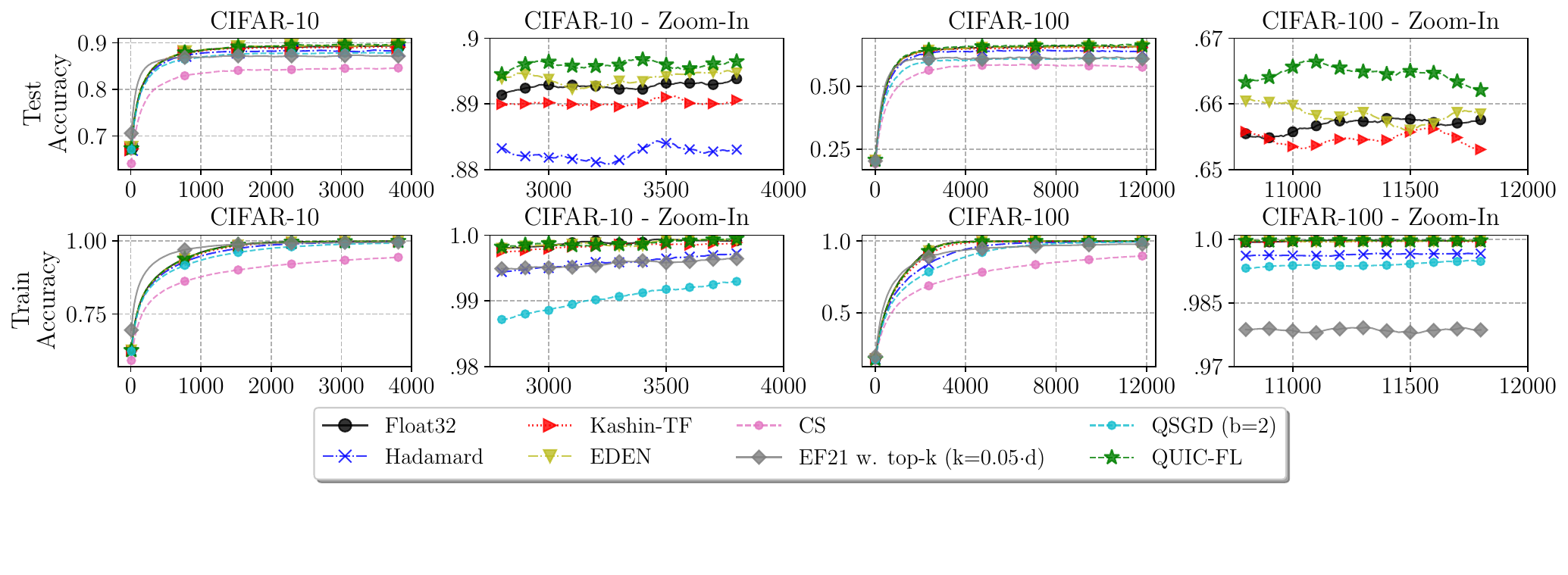}
\caption{Cross-device federated learning.}
\label{fig:cross_device_w_qsgd}
\vspace{2mm}
\end{figure*}



\subsection{DME as a Building Block}\label{app:ef21_topk_qfl}
We pick EF21~\cite{richtarik2021ef21} as an example framework that uses DME as a building block. In the paper, EF21 is used in conjunction with top-$k$ as the compressor that is used by the clients to transmit their messages, and the mean of the messages is estimated at the server. As shown in \cref{fig:EF21QUIC}, using EF21 with \alg instead of top-$k$ significantly improves the accuracy of EF21 despite using less bandwidth. For example, top-$k$ with $k=0.1\cdot d$ needs to use 3.2 bits per coordinate on average to send the values (in addition to the overhead of encoding the indices) while having accuracy that is lower than EF21 with \alg  and $b=2$ bits per coordinate.
\begin{figure*}[h!]
\centering
\includegraphics[trim={0 2.1cm 0 0},clip,width=1\linewidth]{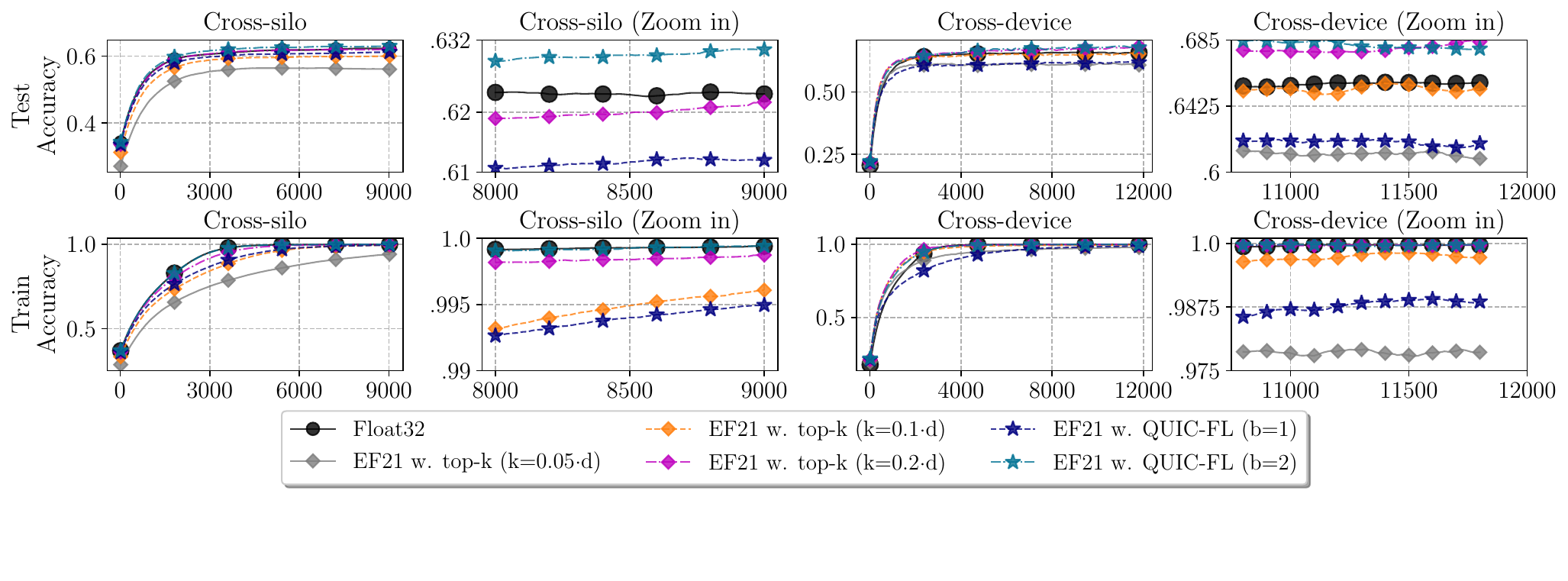}
\caption{The accuracy of EF21 with top-$k$ and \alg as building blocks for DME.}
\label{fig:EF21QUIC}
\end{figure*}
%
%
\subsection{Distributed Power Iteration}\label{app:subsec:pi}

\begin{figure*}[h!]
\centering
\includegraphics[trim={0.25cm 1.9cm 0.25cm 0},clip,width=1\linewidth]{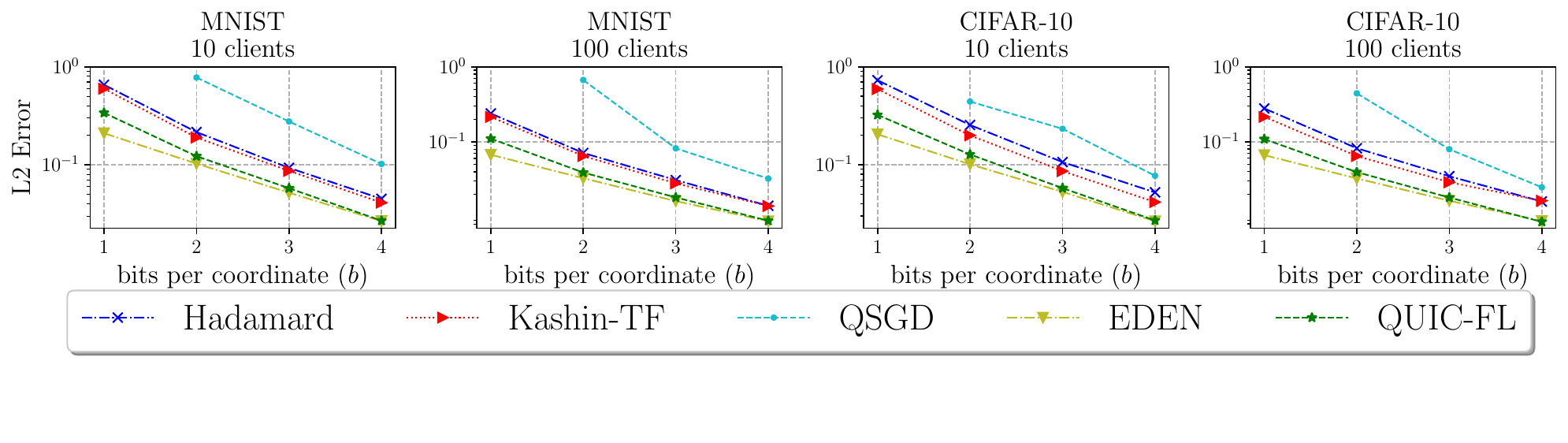}
\caption{Distributed power iteration of MNIST and CIFAR-10 with 10 and 100 \senders.}
\label{fig:pi}
\end{figure*}

\looseness=-1
We simulate $10$ \senders that distributively compute the top eigenvector in a matrix (i.e., the matrix rows are distributed among the \senders). Particularly, each \sender executes a power iteration, compresses its top eigenvector, and sends it to the \receiver. The \receiver updates the next estimated eigenvector by the averaged diffs (of each client to the eigenvector from the previous round) and scales it by a learning rate of $0.1$. Then, the estimated eigenvector is sent by the \receiver to the \senders and the next round can begin. 

Figure~\ref{fig:pi} presents the L2 error of the obtained eigenvector by each compression scheme when compared to the eigenvector that is achieved without compression.
The results cover bit budget $b$ from one bit to four bits for both MNIST and CIFAR-10~\cite{krizhevsky2009learning, lecun1998gradient, lecun2010mnist} datasets.
Each distributed power iteration simulation is executed for 50 rounds for the MNIST dataset \mbox{and for 200 rounds for the CIFAR-10 dataset.}

As shown, \alg has an accuracy that is competitive with that of EDEN (especially for $b\ge2$) while having asymptotically faster decoding, as EDEN requires decompressing the vector for each client independently. 
At the same time, \alg is considerably better in terms of accuracy than other algorithms that offer fast decoding time. Also, Kashin-TF is not unbiased (as illustrated by Figure~\ref{fig:sensitivity}), and is, therefore, less competitive for a larger number of clients.


\begin{figure*}[h!]
\centering
\includegraphics[trim={0 0 0 3cm},clip,width=0.95\linewidth]{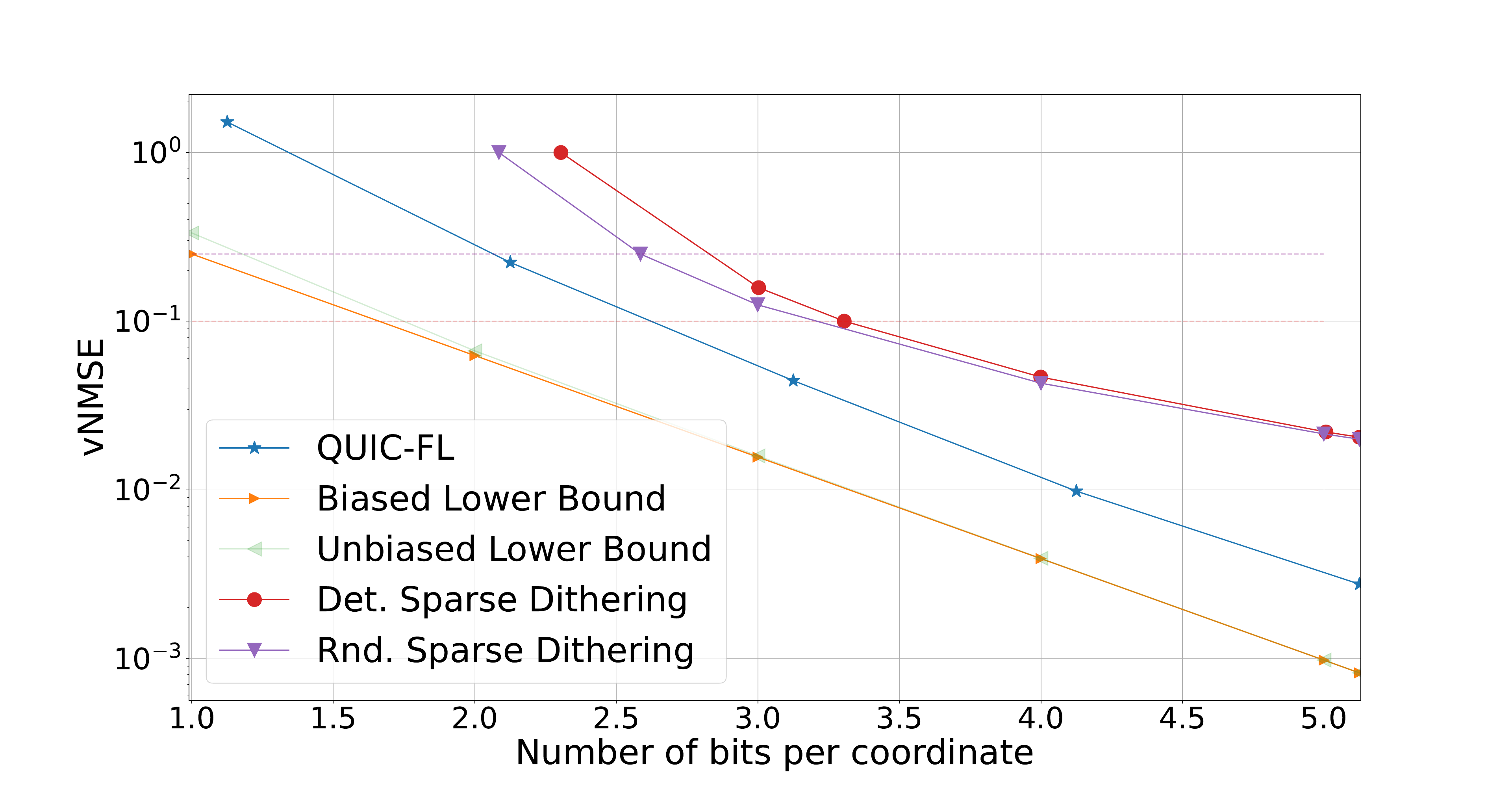}
\vspace{-3mm}
\caption{Comparison with Sparse Dithering.}
\vspace{-2mm}
\label{fig:Compairson_with_SD}
\end{figure*}

\subsection{Comparison with Sparse Dithering}\label{app:sd}
We compare \alg with Sparse Dithering (SD)~\cite{albasyoni2020optimal}.
As shown in Figure~\ref{fig:Compairson_with_SD}, \alg is markedly more accurate for the range of bit budgets ($b\in\set{1,2,3,4,5}$) that it supports. The figure includes both the deterministic and randomized versions of SD.

The markers mark the evaluated points. QUIC-FL is configured with $p=2^{-9}$, and thus its per-coordinate bandwidth is non-integer to factor in the coordinates sent exactly.

Further, our algorithm is proven to be GPU friendly, while we cannot determine whether the components of the Sparse Dithering algorithm can be efficiently implemented. The paper does not include a runtime evaluation that we can compare with.








\end{document}